\documentclass[letterpaper, 10 pt, journal, twoside]{IEEEtran}

\usepackage{microtype}
\usepackage{setspace}
\usepackage{blindtext}
\usepackage{siunitx}
\usepackage{comment}
\usepackage{xr}
\usepackage{titlecaps}
\Addlcwords{or the if a an of for and to -off off in not by via}

\usepackage[font=scriptsize]{caption}
\usepackage[caption=false, font=scriptsize]{subfig}
\usepackage{graphicx}
\usepackage{fancyhdr} 
\usepackage{color}
\usepackage{epsfig}
\usepackage{tikz}
\usepackage{tikzscale}
\tikzstyle{int}=[draw, fill=black!10, minimum size=5em,thick]
\tikzstyle{init} = [pin edge={to-,thick,black}]
\usetikzlibrary{patterns,
	tikzmark,
	arrows,
	arrows.meta,
	plotmarks,
	shapes.arrows,
	spy,
	svg.path,
	shapes.multipart,
	positioning,
	decorations,
	calc,
	decorations.markings,
	patterns,
	fit,
	matrix
}
\tikzset{add reference/.style={insert path={%
			coordinate [pos=0,xshift=-0.5\pgflinewidth,yshift=-0.5\pgflinewidth] (#1 south west) 
			coordinate [pos=1,xshift=0.5\pgflinewidth,yshift=0.5\pgflinewidth]   (#1 north east)
			coordinate [pos=.5] (#1 center)                        
			(#1 south west |- #1 north east)     coordinate (#1 north west)
			(#1 center     |- #1 north east)     coordinate (#1 north)
			(#1 center     |- #1 south west)     coordinate (#1 south)
			(#1 south west -| #1 north east)     coordinate (#1 south east)
			(#1 center     -| #1 south west)     coordinate (#1 west)
			(#1 center     -| #1 north east)     coordinate (#1 east)   
}}}
\usepackage{scalerel}
\usepackage{stackrel}
\usepackage{pgfplots}
\usepgfplotslibrary{fillbetween}
\pgfplotsset{compat=newest}
\pgfplotsset{every axis/.append style={
		label style={font=\Large},
		tick label style={font=\large}  
}}
\usepackage{booktabs}
\usepackage{tcolorbox}
\usepackage{epstopdf}
\graphicspath{{./figures/}}
\usepackage{framed}

\usepackage[ruled,linesnumbered]{algorithm2e}
\usepackage{ifthen}
\usepackage{multirow}
\usepackage{tabularx}

\usepackage{amsmath}
\usepackage{amssymb}
\usepackage{relsize}
\usepackage{nicefrac}
\usepackage{bbm,bm}
  
\usepackage{amsthm}
\usepackage[mathscr]{eucal}
\usepackage[c2,short]{optidef}

\usepackage{array}
\usepackage{paralist}
\usepackage{enumitem}
\usepackage[absolute]{textpos}

\usepackage{url}
\usepackage[colorlinks,linkcolor=blue,citecolor=blue,urlcolor=blue]{hyperref}
\usepackage[capitalize,nameinlink]{cleveref}

\newcommand{\orcid}[1]{\href{https://orcid.org/#1}{\includegraphics[scale=0.04]{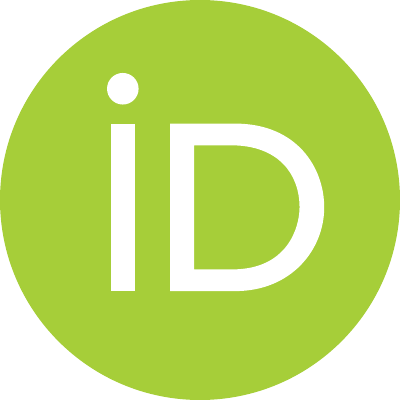}}} 

\newtheorem{thm}{Theorem}
\newtheorem{prob}{Problem}

\newtheorem{ass}[thm]{Assumption}
\newtheorem{definition}[thm]{Definition}

\newtheorem{rem}[thm]{Remark}
\newtheorem{ex}[thm]{Example}

\Crefname{ass}{Assumption}{Assumptions}
\Crefname{prob}{Problem}{Problems}

\newcommand{\bdmath}{\begin{dmath}}
\newcommand{\edmath}{\end{dmath}}
\newcommand{\beq}{\begin{equation}}
\newcommand{\eeq}{\end{equation}}
\newcommand{\bdm}{\begin{displaymath}}
\newcommand{\edm}{\end{displaymath}}
\newcommand{\bea}{\begin{eqnarray}}
\newcommand{\eea}{\end{eqnarray}}
\newcommand{\beal}{\beq \begin{array}{ll}}
\newcommand{\eeal}{\end{array} \eeq}
\newcommand{\beas}{\begin{eqnarray*}}
\newcommand{\eeas}{\end{eqnarray*}}
\newcommand{\ba}{\begin{array}}
\newcommand{\ea}{\end{array}}
\newcommand{\bit}{\begin{itemize}}
\newcommand{\eit}{\end{itemize}}
\newcommand{\ben}{\begin{enumerate}}
\newcommand{\een}{\end{enumerate}}

\newcommand{\lr}{\left(}
\newcommand{\rr}{\right)}

\newcommand{\lb}{\left\lbrace}
\newcommand{\rb}{\right\rbrace}

\newcommand{\calE}{{\cal E}}

\newcommand{\calG}{{\cal G}}

\newcommand{\calI}{{\cal I}}

\newcommand{\calK}{{\cal K}}
\newcommand{\calL}{{\cal L}}
\newcommand{\calM}{{\cal M}}

\newcommand{\calP}{{\cal P}}

\newcommand{\calR}{{\cal R}}
\newcommand{\calS}{{\cal S}}
\newcommand{\calT}{{\cal T}}
\newcommand{\calU}{{\cal U}}
\newcommand{\calV}{{\cal V}}

\newcommand{\calX}{{\cal X}}

\newcommand{\eg}{\emph{e.g.,}\xspace}
\newcommand{\ie}{\emph{i.e.,}\xspace}
\newcommand{\cf}{\emph{cf.}\xspace}
\newcommand{\myParagraph}[1]{{\bf \titlecap{#1}.\xspace}}

\renewcommand{\boldsymbol}[1]{{\bm #1}}

\newcommand{\hide}[1]{}

\newcommand{\hiddenText}{{\color{gray} hidden text.}}
\newcommand{\hideWithText}[1]{\hiddenText}

\DeclareMathOperator*{\sgn}{sgn}

\newcommand{\tran}{^{\top}}

\newcommand{\Real}[1]{\mathbb{R}^{#1}}
\newcommand{\reals}{\Real{}}

\newcommand{\norm}[2][]{\left\| #2 \right\|_{#1}}

\newcommand{\vf}{\boldsymbol{f}}
\newcommand{\vg}{\boldsymbol{g}}

\newcommand{\vp}{\boldsymbol{p}}

\newcommand{\vu}{\boldsymbol{u}}

\newcommand{\vxx}{\boldsymbol{x}}
\newcommand{\vy}{\boldsymbol{y}}
\newcommand{\vw}{\boldsymbol{w}}
\newcommand{\vzz}{\boldsymbol{z}}

\newcommand{\blue}[1]{{\color{blue}#1}}

\newcommand{\revision}[1]{#1}

\newcommand{\linkToPdf}[1]{\href{#1}{\blue{(pdf)}}}
\newcommand{\linkToPpt}[1]{\href{#1}{\blue{(ppt)}}}
\newcommand{\linkToCode}[1]{\href{#1}{\blue{(code)}}}
\newcommand{\linkToWeb}[1]{\href{#1}{\blue{(web)}}}
\newcommand{\linkToVideo}[1]{\href{#1}{\blue{(video)}}}
\newcommand{\linkToMedia}[1]{\href{#1}{\blue{(media)}}}
\newcommand{\award}[1]{\xspace}

\newcommand{\isExtended}[2]{#1} 
\newcommand{\remove}[1]{}

\newcommand{\pos}[1]{\ensuremath{\left[ #1\right]^{+}}\xspace}

\DeclareMathOperator*{\aggr}{\odot}

\newcommand{\numRobots}{\ensuremath{R}\xspace}
\newcommand{\robset}{\calR}

\newcommand{\x}[2][\@empty]{\vxx_{#2}%
	\ifx\@empty#1 \else (#1) \fi}
\newcommand{\xdot}[2][\@empty]{\dot{\vxx}_{#2}%
	\ifx\@empty#1 \else (#1) \fi}
\newcommand{\xhat}[2][\@empty]{\hat{\vxx}_{#2}%
	\ifx\@empty#1 {} \else {(#1)} \fi}
\newcommand{\xn}[2][\@empty]{\vw_{#2}%
	\ifx\@empty#1 {} \else {(#1)} \fi}
\newcommand{\safe}{\calS}
\newcommand{\unsafe}{\calS_u}
\newcommand{\neigh}[2][\@emtpy]{\mathcal{N}_{#2}%
	\ifx\@emtpy#1 \else (#1) \fi}
\newcommand{\neighrec}[2][\@emtpy]{\mathcal{N}_{#2}^\text{rec}%
	\ifx\@emtpy#1 \else (#1) \fi}
\newcommand{\xt}[2][]{\vxx_{#2,#1}}
	
\renewcommand{\u}[2][\@empty]{\vu_{#2}%
	\ifx\@empty#1 \else (#1) \fi}
\newcommand{\un}[2][\@empty]{\vu_{\neigh{#2}}^%
	\ifx\@empty#1 {} \else {#1} \fi}
\newcommand{\uset}{\calU}
\newcommand{\uref}[2][\@empty]{\bar{\vu}_{#2}%
	\ifx\@empty#1 {} \else {(#1)} \fi}
\newcommand{\upert}[2][\@empty]{\tilde{\vu}_{#2}%
	\ifx\@empty#1 {} \else {(#1)} \fi}
\newcommand{\uperth}[2][\@empty]{\tilde{\vu}_{#2}^\text{h}%
	\ifx\@empty#1 {} \else {(#1)} \fi}
\newcommand{\uhist}[3][\@empty]{u_{#2:#3}^%
	\ifx\@empty#1 {} \else {(#1)} \fi}
\newcommand{\ut}[2][]{\vu_{#2,#1}}

\newcommand{\delays}[2][]{\delta_{#1}(#2)}
\newcommand{\aoi}[3]{\Delta_#1^#2(#3)}
\newcommand{\aoimax}{\Delta_\text{max}}
\newcommand{\srate}{S_\text{rate}}

\newcommand{\data}[2]{\calI_{#1\leftarrow#2}}
\newcommand{\datarel}[2]{\calI_{#1\leftarrow#2}^#1}

\newcommand{\mess}[3]{I_{#1}^{#2}(#3)}
\newcommand{\hist}[2]{\calT_{#1\leftarrow#2}}

\newcommand{\mlp}{m}
\newcommand{\gnn}{\Gamma}
\newcommand{\feat}[3][\@empty]{f_{#2}^{#3}%
	\ifx\@empty#1 \else (#1) \fi}
\newcommand{\featedge}[2]{e_{#1,#2}}

\newcommand{\paramcbf}{\theta}
\newcommand{\cbfmodel}{h_\paramcbf}
\newcommand{\cbfval}[2][]{h_{#1}(#2)}
\newcommand{\cbfdot}[2][]{\dot{h}_{#1}(#2)}

\newcommand{\losscbf}[1][]{\calL_{\text{CBF},#1}}

\newcommand{\paramcontr}{\xi}
\newcommand{\controlmodel}{\pi_\paramcontr}
\newcommand{\controller}{\pi_\text{safe}}
\newcommand{\controllernom}{\pi_\text{nom}}

\newcommand{\control}[1]{\controlmodel(#1)}
\newcommand{\losscontr}[1][]{\calL_{\text{contr},#1}}
\newcommand{\step}{T_\text{s}}

\newcommand{\parampred}{\zeta}
\newcommand{\predmodel}{\lambda_\parampred}
\newcommand{\pred}[1]{\calP\left(#1\right)}
\newcommand{\predval}[2]{\widehat{\vw}_{#1}(#2)}
\newcommand{\predl}[3]{\Psi_{\text{L}}%
	\ifx\\#1\\ {} \else {\lr#1,#2,#3\rr} \fi}
\newcommand{\losspred}{\calL_\text{pred}}

\IEEEoverridecommandlockouts

\title{Safe Distributed Control of Multi-Robot Systems With Communication Delays}

\author{Luca~Ballotta\textsuperscript{\orcid{0000-0002-6521-7142}}
	and Rajat~Talak\textsuperscript{\orcid{0000-0002-6132-395X}},~\IEEEmembership{Member,~IEEE}
	\thanks{This work was supported in part by ARL DCIST CRA under Grant W911NF-17-2-0181
		and by the CARIPARO Foundation Visiting Programme ``HiPeR''.
	}
	\thanks{Luca Ballotta is with the Delft Center for Systems and Control (DCSC), Delft University of Technology, 2628 CD Delft, The Netherlands
		(e-mail: l.ballotta@tudelft.nl).}%
	\thanks{Rajat Talak is with the Laboratory for Information and Decision Systems (LIDS), Massachusetts Institute of Technology, Cambridge, MA 02139, USA
		(e-mail: talak@mit.edu).}
}

\begin{document}
	
	\maketitle
	
	\begin{textblock}{20}(-2,0.1)
		\footnotesize
		\centering
		\setstretch{1}
		This article has been accepted for publication in the IEEE Transactions on Vehicular Technology.\\
		Please cite the article as: L. Ballotta and R. Talak,\\
		``Safe Distributed Control of Multi-Robot Systems With Communication Delays,''\\
		IEEE Transactions on Vehicular Technology, 2025.\\
	\end{textblock}
	

\begin{abstract}
	Safe operation of multi-robot systems is critical, 
	especially in communication-degraded environments such as underwater for seabed mapping,	underground caves for navigation,
	and in extraterrestrial missions for assembly and construction. 
	We address safety of networked autonomous systems where the information exchanged between robots incurs communication delays.
	We formalize a notion of \emph{distributed control barrier function} for multi-robot systems,
	a safety certificate amenable to a distributed implementation,
	which provides formal ground to using graph neural networks to learn safe distributed controllers.
	Further,
	we observe that learning a distributed controller ignoring delays can severely degrade safety.
	We finally propose a predictor-based framework to train a safe distributed controller under communication delays,
	where the current state of nearby robots is predicted from received data and age-of-information.
	Numerical experiments on multi-robot collision avoidance show that our predictor-based approach can significantly improve the safety of a learned distributed controller under communication delays.
	\revision{A video abstract is available at \url{https://youtu.be/Hcu1Ri32Spk}.}
	
	\begin{IEEEkeywords}
		Communication delays,
		distributed control barrier function, 
		graph neural network,
		multi-robot system,
		safety.
	\end{IEEEkeywords}
	
\end{abstract}

\section{Introduction}
\label{sec:intro}

\IEEEPARstart{M}{obile} autonomous robot networked systems are increasingly being conceived to aid humans in oceanbed mapping, 
underground subterranean navigation, 
search-and-rescue missions,
and space exploration~\cite{Hu21tvt-coordinatedControlMobileRobots,
	Li24tvt-flockingMultiRobot,
	Zhou19tvt-rlMultiRobot,
	Hu20tvt-voronoiMultiRobot,
	Lee13tm-teleoperation}.
Safe and coordinated operation of multi-robot systems is critical to their successful deployment.
However,
operation environments often induce severe communication outages that make information exchange between robots imperfect and delayed~\cite{Li24tvt-flockingMultiRobot,Pezzutto22lcss-remoteMPC,Capelli21icra-connectivityCBFDelays}.
Moreover,
the safety requirements of such networked autonomous system depend on of all robots,
and not just one (\eg robots need to avoid collisions with each other or move in a formation, without breaking connectivity).
Therefore, there is a need to address the question of safety in conjunction with imperfect information-exchange between robots in a networked autonomous system.

A recent trend has seen control barrier functions (CBFs) as a promising tool that ensures safety of control actions by design~\cite{Ames17tac-cbf}.
While centralized CBF-based control theoretically ensures safety,
its use is impractical for networked autonomous systems where decentralized (no communication) or distributed (inter-robot communication) control is preferred.
Previous work has addressed decentralized and distributed safe control with handcrafted CBFs~\cite{Wang17tro-safeMultirobotCBF},
which show promising results towards network scalability.
However,
computing a valid CBF for multi-robot and networked systems may be computationally hard,
and safety degradation of CBF-based control under communication delays has received little attention so far.

\begin{figure}
	\centering
	\includegraphics[width=.7\linewidth]{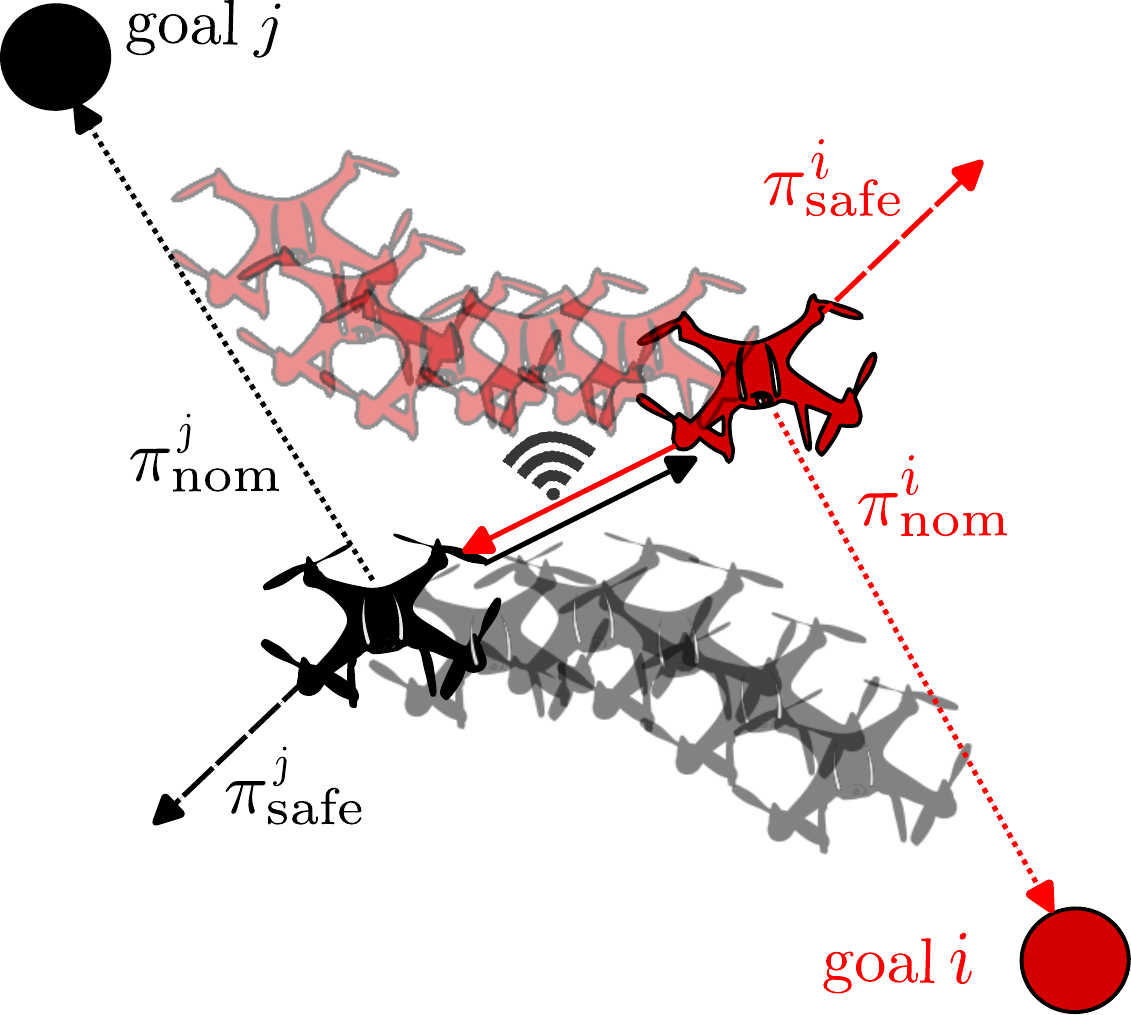}
	\caption{\revision{We propose a control strategy that keeps an autonomous multi-robot system safe via inter-robot communication.
		In this example,
		the controllers $\controller^i$ and $\controller^j$ of drones $i$ and $j$ use data sent by the other drone to avoid collisions while reaching the goals.}
	}
	\label{fig:cover-1}
\end{figure}

\begin{figure}
	\centering
	\includegraphics[width=.7\linewidth]{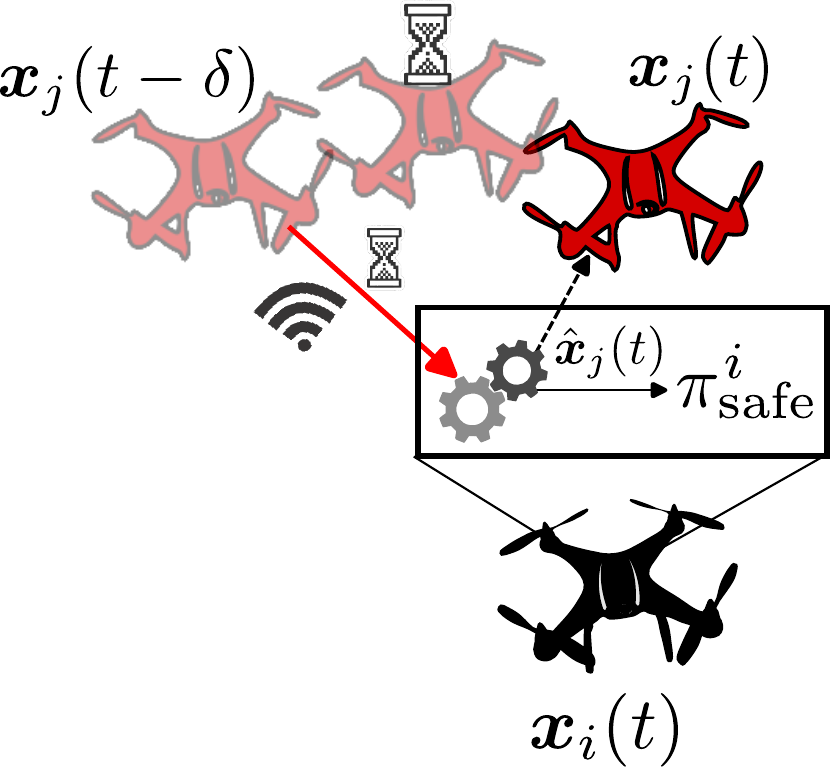}
	\caption{\revision{Safe control via wireless communication is crucially affected by transmission delays.
		We propose a predictor that uses delayed measurement $\x[t-\delta]{j}$ to compute the estimate $\xhat[t]{j}$ that is used by the controller $\controller^i$ of robot $i$ for safe real-time operation.}
	}
	\label{fig:cover-2}
\end{figure}

\myParagraph{Contribution}
We first propose a formal framework for CBF-based safe distributed control.
We give a tailored definition of safe set for networked autonomous systems in \autoref{sec:setup} and formally characterize a \textit{distributed control barrier function} that guarantees safety under local communication in \autoref{sec:distr-cbf-perfect-info}.
This theoretically justifies a distributed mechanism for safety certification and a \textit{safe distributed controller} (\autoref{fig:cover-1}).
Our framework extends and complements previous work~\cite{Chen21lcss-cbfMultiAgent} that provides a pairwise decoupling condition on the CBF for a safe decentralized controller with backup control.
\isExtended{}{Moreover,
we highlight connections that enable our theoretical framework for robotic applications.}

Secondly,
we present a learning-based approach that builds on our formal results and leverages graph neural networks (GNNs) in \autoref{sec:distr-cbf-gnn}.
We consider a distributed controller where each robot computes its own control inputs based on data received from neighbors.
This allows for flexible controllers that can scale to larger robot teams and require less computation and communication than distributed optimization.

Our final contribution in \autoref{sec:distr-cbf-realistic-information} is a novel predictor-based safe distributed control framework under communication delays.
We consider delays that affect information exchange between robots,
differently from previous work on CBF that addresses actuation delays~\cite{Jankovic18acc-cbfInputDelay} or tackles this issue via heuristics~\cite{Capelli21icra-connectivityCBFDelays}.
\revision{Indeed,
	communication delays are more challenging because the receiving robot must estimate the states of other robots without knowing their past inputs,
	while input delays can be compensated through the robot's dynamics}.
We propose a prediction-based framework that is agnostic to the safety task to handle delays.
In it,
a predictor implemented at each robot predicts the current states of nearby robots given previous state information received from them and the associated Age-of-Information~\cite{Sun19book-AgeInformation},
which captures \emph{staleness} of received data (\autoref{fig:cover-2}).
We show that a simple alternating strategy, 
which alternatively learns a predictor model (given controller) and a controller model (given predictor), 
works. 

Our experiments in \autoref{sec:experiments} show that the GNN-based controller can safely navigate a team of mobile robots to avoid collisions under perfect communication,
and that the predictor-based controller significantly improves safety in the realistic case when transmissions incur communication delays.

\revision{To the best of our knowledge,
	this is the first work that formalizes a general distributed control barrier function suited to safe control based on inter-agent communication.
	Also,
	it is the first that proposes a general CBF-based control architecture to tackle communication delays between agents.}

\subsection{Related Work}

\myParagraph{Certified Control}
Recent years have seen a surge in attention towards \emph{control certificates},
especially control barrier functions (CBF) to meet hard constraints in the state space,
often interpreted as \emph{safety}.
Control Barrier Functions were introduced in~\cite{Ames17tac-cbf} inspired by barrier functions in optimization.
Follow-up work extended them to uncertain dynamic environments~\cite{Molnar23tcst-cbfInputDelay,Hamdipoor23ejc-environmentallyRobustCBF},
online adaptation~\cite{Breeden22cdc-proactiveCBF},
and observer-controller co-design~\cite{Agrawal23lcss-observerControllerCBF},
to make a few examples.
Decentralized collision avoidance via CBF was studied in~\cite{Wang17tro-safeMultirobotCBF,Chen21lcss-cbfMultiAgent} \revision{without communication between robots}.
Previously,
a similar effort was made in~\cite{Panagou16tac-multiRobotLyapunovBarrierFunction} for leader-follower distributed control where ad-hoc barrier functions were proposed for several control objectives.
Other relevant works are~\cite{Cavorsi22rss-resilienceCBF},
where a resilient algorithm was proposed to tame adversarial agents,
and~\cite{Ong23automatica-nonsmoothCBFConnectivityMaintenance},
that addresses connectivity maintenance based on Laplacian eigenvalues.

\myParagraph{Learning-based certificates}
Besides optimization-based methods such as SoS,
a recent trend to find a CBF is to use a machine learning model such as a neural network (NN).
Reference~\cite{Xiao23tro-BarrierNet} proposed BarrierNet,
a learning-based tool that tunes parameters to improve performance while guaranteeing safety.
Paper~\cite{Gaby22cdc-Lyapunov-Net} uses an ad-hoc design to learn Lyapunov functions with guaranteed positive semi-definiteness.
Works~\cite{Abate21lcss-formalSynthesisLyapunovNN,Abate21ichs-FOSSIL} introduced FOSSIL,
which validates the learned CBF via Satisfiability Modulo Theory (SMT).
While this approach is general,
the formal verification is computationally intractable for large model or state dimension.
Works~\cite{Zakwan23lcss-hamiltonianNN,Furieri22cdc-neuralSystemLevelSynthesis} studied learnable controllers with formal guarantees implicitly provided by the considered class of closed-loop systems.
The authors in~\cite{Mathiesen23lcss-neuralBarrierStochastic} proposed a learning-based approach with formal verification based on linear approximations of neural networks for stochastic closed-loop dynamics.

\myParagraph{Control with Delays}
In distributed control,
wireless communication affects feedback information in terms of noise and delays.
These issues were tackled in~\cite{Sinopoli04tac-KalmanIntermittent,
	Munz10automatica-consensusDelays,
	Matni17tcns-h2controlAtomicNorm,
	Gomez19tac-h2controlDelays,
	Branz22tcst-Drive-by-Wi-Fi},
to name a few.
Delays in optimal control design are challenging because even fairly realistic models induce nonconvex or intractable optimization problems.
The communication community has proposed Age-of-Information~\cite{Sun19book-AgeInformation},
with a large body of works addressing information freshness~\cite{Tripathi23tmc-AgeOptimal,Talak20tn-aoiInterference}
or control-oriented metrics~\cite{Tripathi19allerton-whittle,
	Ayan19-AoIvsVoI,
	Klugel19infocom-aoiNetworkedControl,
	Champati19infocom-aoiNetworkSingleLoop}.
In these works,
the design focuses on scheduling of updates.
Also,
CBF-related research has addressed delays.
Input delays were counterbalanced in~\cite{Jankovic18acc-cbfInputDelay,Molnar23tcst-cbfInputDelay} with a predictor.
Delays affecting input, dynamics, and CBF were addressed in~\cite{Kiss23ijrnl-cbfDelays},
with Control Barrier Functionals for retarded systems.
However,
knowledge of the state at the current time is assumed,
whereas \revision{realistic} communication delays cause each robot to know only past states of other robots.
A heuristic CBF-based control was proposed in~\cite{Capelli21icra-connectivityCBFDelays} for connectivity maintenance and collision avoidance \revision{for single integrator dynamics and all-to-all communication}.

\myParagraph{Graph Neural Networks for Control}
GNNs are general neural network models that inherently learn distributed architectures~\cite{Kipf16iclr-graphCOnvolutonalNetworks,
	Velickovic18iclr-attentionNetwork,
	Brody21iclr-howAttentive}.
Early works~\cite{Khan20corl-graphPolicyGradient,Li20iros-gnnDecentralizedControl} focused on applicability of Graph Convolutional Networks to distributed control.
Follow-up research~\cite{Kortvelesy21icra-ModGNN} proposed more general models,
paper~\cite{Gama22tsp-gnnCOntroller} further broadened the framework including imitation and transfer learning,
and~\cite{Sebastian23icra-LEMURS} presented a physics-informed distributed controller for port-Hamiltonian systems.
However,
the works above do not address safety based on CBFs and neglect delays in information exchange.
		\revision{

\subsection{Notation}
\label{sec:notation}

Normal font, boldface (\eg $\vxx$),
and calligraphic letters (\eg $\safe$)  denote scalars, vectors, and sets,
respectively.
The symbol $\ominus$ denotes difference between two elements in an Euclidean or non-Euclidean space.
The symbol $\odot$ denotes a permutation-invariant operation over the elements of a finite set.}
	

\section{Background}
\label{sec:background}

This section reviews the main tools used in our framework.

\subsection{Control Barrier Functions}
\label{sec:cbf-safety}

The state of a controlled dynamical system evolves as 
\begin{equation}
	\label{eq:dynamics-general}
	\dot{\vxx}(t) = \vf(\x[t]{}, \vu(t)),
\end{equation}
where $\x[t]{}\in\calX$ denotes the system state at time $t$,
$\vu(t)\in\uset$ is the control action at time $t$,
and $\vf$ is a locally Lipschitz function.
Being us concerned with control design,
the control input $\u[t]{}$ is computed by a controller $\pi$.
In the following,
when there is no risk of ambiguity,
we occasionally omit the dependence on time $t$.

The system is \emph{safe} whenever the state lies in the \emph{safe set} $\safe\subset\calX$,
\ie $\vxx \in \safe$. 
The safe set is task dependent and describes safety in the state space.
For example,
if safety means avoiding collisions,
the set $\safe$ contains all collision-free configurations.
Under this setting,
the goal is to develop a certification mechanism that can attest whether the controlled dynamical system,
starting from the safe set,
remains in it forever.
If this happens when the state $\vxx$ obeys dynamics~\eqref{eq:dynamics-general},
we say that $\safe$ is \emph{forward invariant} w.r.t.~\eqref{eq:dynamics-general}.
The set $\safe$ is typically assumed compact.
The unsafe set is $\unsafe \doteq \calX \setminus \safe$. 

Control barrier functions (CBFs) provide an elegant way to certify whether the controlled dynamical system will remain in the safe state forever,
or not.
In this section,
we briefly recap the main results on CBF-based safety guarantees and safe control.
The interested reader is referred to~\cite{Ames17tac-cbf} for details.

Assume that the safe set can be described as the superlevel set of a continuous function $h$, 
as follows:
\begin{equation}\label{eq:safe-set-cbf}
	\safe = \{\vxx \in\calX : h(\vxx)\ge0\}.
\end{equation}
\remove{\begin{itemize}
	\item $\safe = \{\vxx \in\calX : h(\vxx)\ge0\}$\remove{$\mathrm{int}(\safe) = \{\vxx \in\calX : h(\vxx)>0\}$};
	\item $\unsafe = \{\vxx \in\calX : h(\vxx)<0\}$\remove{$\partial\safe = \{\vxx \in\calX : h(\vxx)=0\}$}.
\end{itemize}}

\isExtended{}{
	\begin{rem}
		These definitions of safe and unsafe sets ensure that the boundary of the safe set $\partial\safe$ is the $0$-level set of $h$.
	\end{rem}
}

\begin{definition}[Control Barrier Function~\cite{Ames17tac-cbf}]
	\label{def:cbf}
	Given a set $\safe$ defined above for a continuously differentiable function $h:\calX \rightarrow \reals$,
	the function $h$ is a \emph{Control Barrier Function} (CBF) for system~\eqref{eq:dynamics-general} on $\calM\subset\calX$ with $\safe\subseteq\calM$ if there exists an extended class $\calK$-infinity function such that
	\begin{equation}\label{eq:cbf-condition}
		\sup_{\vu\in\calU} \; \cbfdot{\vxx,\u{}} \ge -\alpha(\cbfval{\vxx}) \quad \forall \vxx\in\calM.
	\end{equation}
	In this case,
	the \emph{safe control set} associated with state $\x{}$ is
	\begin{equation}
		\label{eq:thm:cbf}
		\calU_\safe(\vxx) \doteq \left\{ \vu \in \calU : \cbfdot{\vxx,\vu} \ge -\alpha(h(\vxx)) \right\}.
	\end{equation}
\end{definition}

\begin{thm}[Safe control~\cite{Ames17tac-cbf}]
	\label{thm:cbf}
	If $h$ is a CBF for~\eqref{eq:dynamics-general},
	any Lipschitz continuous controller $\pi:\calM\to\calU$ s.t. $\pi(\vxx)\in\calU_\safe(\vxx)$ makes the safe state $\safe$ forward invariant w.r.t.~\eqref{eq:dynamics-general}.
\end{thm}

\begin{rem}[Time derivative of CBF]
	The time derivative of the CBF $\cbfval{\vxx}$,
	$\cbfdot{\vxx,\vu}$,
	is given by
	\begin{equation}
		\cbfdot{\vxx,\vu} = L_\vf h(\vxx,\vu) = \nabla h(\vxx)\tran\vf(\vxx, \vu)
	\end{equation}
	where $L_\vf h$ denotes the Lie derivative of $h$ with respect to $\vf$.
\end{rem}

\begin{rem}[CBF with distributed communication]
	Crucially,
	while the CBF $\cbfval{\vxx}$ depends on the state $\vxx$,
	its time derivative $\cbfdot{\vxx,\vu}$ depends on both state $\vxx$ and control action $\vu$ through the state dynamics~\eqref{eq:dynamics-general}.
	In a distributed communication context,
	if a unique CBF is used for the whole multi-agent system,
	this means that all agents (robots) should instantaneously know the states of all other agents to compute $\cbfval{\vxx}$,
	and both states and control actions of all other agents to compute $\cbfdot{\vxx,\vu}$.
	In practice,
	this may be an infeasible requirement.
	In \autoref{sec:distributed-cbf},
	we build a learning-based \emph{distributed CBF} framework that allows robots to \emph{locally} coordinate in order to achieve safety.
\end{rem}

\myParagraph{Implementing a Safe Controller}
Given 
a CBF $h$, 
\cref{thm:cbf} offers a way to compute the control input $\u{}$ in~\eqref{eq:dynamics-general} so as to keep the state inside the set $\safe$.
A common situation is where a nominal controller $\controllernom$ is already available and designed to optimize a task of interest,
such as trajectory tracking or goal reaching,
without guaranteeing safety.
We denote nominal control inputs computed by $\controllernom$ as $\uref{}$.
In this case,
one can use the CBF as a filter to design a new control policy that ensures safety while being minimally deviant from the nominal controller. 
This can be achieved by solving the following optimization problem at each time $t$:
\begin{mini!}
	{\vu\in\calU}
	{\|\vu-\uref[t]{}\|^{2}\protect\label{eq:control-optimization-based-objective}}
	{\label{eq:control-optimization-based}}
	{\u[t]{} = }
	\addConstraint{\cbfdot{\x[t]{},\vu}}{\ge-\alpha(\cbfval{\x[t]{}}).\protect\label{eq:control-optimization-based-CBS}}
\end{mini!}
If a nominal controller is not available,
one can still use~\eqref{eq:control-optimization-based} with $\uref[t]{}\equiv0$ to minimize control effort,
an approach usually called \emph{point-wise minimum-norm control}.

\begin{rem}[Convexity]
	The optimization problem~\eqref{eq:control-optimization-based} becomes much simpler when dynamics~\eqref{eq:dynamics-general} are control-affine, 
	\ie
	\begin{equation}
		\xdot[t]{} = \vf(\x[t]{}) + \vg(\x[t]{})\u[t]{},
	\end{equation}	
	for some locally Lipschitz $\vf$ and $\vg$,
	and the set $\uset$ is a polytope.
	In this case,
	problem~\eqref{eq:control-optimization-based} is a QP and can be efficiently solved. 
	Moreover,
	problem~\eqref{eq:control-optimization-based} remains convex if $\uset$ is convex.
\end{rem}


\subsection{Learning Control Barrier Functions}
\label{sec:cbf-learning}

Because the state-of-the-art techniques to find a CBF $h$ that meets~\eqref{eq:safe-set-cbf}--\eqref{eq:cbf-condition} rely on Sum-of-Squares programming~\cite{Ames17tac-cbf},
the computational complexity becomes intractable for large state dimension.
Recent works have shown that CBFs can be computed via machine learning~\cite{Abate21ichs-FOSSIL,
	Furieri22cdc-neuralSystemLevelSynthesis,
	Mathiesen23lcss-neuralBarrierStochastic}.
%
In this case,
$h$ is parametrized by a trainable model $\cbfmodel$ (\eg a neural network) that is trained to meet the conditions that characterize a CBF. 
This can be done by minimizing the loss $\calL(\paramcbf)$ defined as
\begin{multline}\label{eq:loss-cbf-background}
	\calL(\paramcbf) \doteq 
	\sum_{\vxx \sim \calS} \pos{\epsilon - \cbfval[\paramcbf]{\x{}}} 
	+ \sum_{\vxx \sim \calS_u} \pos{\epsilon + \cbfval[\paramcbf]{\x{}}} \\
	+ \sum_{\vxx \sim \calX} \pos{\epsilon - \dot{\cbfmodel}(\x{}, \u{}) - \alpha(\cbfval[\paramcbf]{\x{}})},
\end{multline}
where $\pos{\cdot}\doteq\max\{\cdot,0\}$ and the summations are computed over the states sampled during training, 
while control actions are selected from the set $\calU_\safe(\vxx)$, 
\eg via~\eqref{eq:control-optimization-based}.
The slack hyperparameter $\epsilon>0$ softly enforces the CBF conditions on states ``close'' to training samples by compactness of $\safe$.

Recent works have also proposed to learn a safety certificate and a safe controller simultaneously. 
The controller is parametrized as a second learning-based model, 
say $\controlmodel$,
and the two models are jointly trained with the loss $\calL(\paramcbf, \paramcontr)$ defined as in~\eqref{eq:loss-cbf-background} but with the control input $\u{}$ replaced by the learned control $\control{\x{}}$.
%
To stay ``close'' to a nominal controller $\controllernom$,
one can add the deviation term $\norm{\control{\x{}} - \uref{}}$ to the loss $\calL(\paramcbf, \paramcontr)$.

\subsection{Graph Neural Networks}
\label{sec:back-gnn}

Graph neural networks are machine learning models for graph-based data.
Each node in the graph carries one copy of a common GNN model,
which takes in input data associated with that node and with its neighbors (\emph{node features}) plus data shared from neighbors to the node (\emph{edge features}).

The most general implementations use so-called neighborhood aggregation or message passing.
Given initial node features $\feat{i}{(0)}$ of node $i$ and edge features $\featedge{i}{j}$ from in-neighbor node $j$ to node $i$,
the features of node $i$ in each {layer} $k=0,\dots,K-1$ are iteratively updated as
\begin{equation}\label{eq:message-passing}
	\feat{i}{(k+1)} = \gamma^{(k)} \lr \feat{i}{(k)}, \aggr_{j \in \neigh{i}} \phi^{(k)}\lr\feat{i}{(k)},\feat{j}{(k)},\featedge{j}{i}\rr \rr,
\end{equation}
where $\phi$ and $\gamma$ are differentiable functions (usually parameterized as MLPs),
$\aggr$ is a differentiable,
permutation-invariant aggregation function (\eg sum or max),
and $\neigh{i}$ collects in-neighbors of node $i$ in the message-passing graph.
The output of the GNN for node $i$ is given by features $\feat{i}{(K)}$ computed in the last layer.
	

\section{Setup}\label{sec:setup}

In this section,
we first describe the system model including robot dynamics,
information exchange between robots,
and safety of the networked multi-robot system in \autoref{sec:system-model}.
Then,
we state the problem about safe distributed control design in \autoref{sec:problem-statement},
which is addressed in the rest of the article.

\subsection{System Model}
\label{sec:system-model}

\myParagraph{Robot Dynamics}
We consider a multi-robot system where $R$ robots fulfill a control task.
We label each robot with index $i\in\robset\doteq\{1,\dots,R\}$.
The dynamics of robot $i$ are given by
\begin{equation}
	\label{eq:dynamics}
	\xdot[t]{i} = \vf(\x[t]{i}, \u[t]{i}),
\end{equation}
where $\x[t]{i} \in \calX$ denote the state of the robot and $\u[t]{i}\in\uset$ the control action, at time $t$. 
The state of the robot may include a physical description such as its position and velocity,
among other things. 
We use $\vp_i\in\Real{n}$ to denote the position of robot $i$,
with $n\in\{2,3\}$.
We assume that all robots have the same dynamics $\vf$ and that the state and action sets $\calX$ and $\uset$ are identical for all robots.
Vectors $\x{}\in\calX^R$ and $\u{}\in\calU^R$ stack respectively states and control actions of all robots.

\myParagraph{Information Exchange}
Robots can communicate through a wireless channel.
Within a \emph{distributed control} setting,
local information exchange between robots is used to achieve coordination and cooperation by computing control inputs that suitably take into account the configurations and actions of other robots.
Let $R_c$ be the communication radius. 
Two robots within communication range (\ie $\norm{\vp_i - \vp_j} \le R_c$) can exchange information,
\eg their states.
We use the symbol $\mess{j}{}{t}$ to denote the \emph{message} transmitted by robot $j$ to robots within range at time $t$, 
and describe it in detail later.
Under \emph{perfect-information exchange}, 
a robot instantaneously receives information from all the other robots within its communication radius 
(\eg if $\norm{\vp_i(t) - \vp_j(t)} \leq  R_c$,
then robot $i$ knows $\vxx_j(t)$ at time $t$).
However, a \emph{realistic information-exchange} means that robots can have only delayed information of other robots within communication range,
due to interference and wireless channel fading.
We denote all time-epochs when the information was received by $i$ from $j$ till time $t$ as
\begin{equation}\label{eq:hist}
	\hist{i}{j}(t) \doteq \lb t' \le t : \mbox{robot } i \mbox{ receives } \mess{j}{}{t'} \rb
\end{equation}
and the information that has been received by robot $i$ from robot $j$ till time $t$ by
\begin{equation}
	\label{eq:info-ij-all}
	\data{i}{j}(t) \doteq \left\{ \mess{j}{}{t'} : t' \in \hist{i}{j}(t) \right\}.
\end{equation}
Under perfect-information exchange, 
the set $\hist{i}{j}(t)$ in~\eqref{eq:hist} corresponds to all times when the robots $i$ and $j$ were within the communication radius, \ie 
\begin{equation}
	\hist{i}{j}(t) = \lb t' \leq t : \norm{\vp_i(t') - \vp_j(t')} \leq R_c \rb.
\end{equation}
However,
under realistic information exchange,
not all information sent by robot $j$ may be received by robot $i$ and $\calI_{i \leftarrow j}(t)$ contains only information that has been successfully received.
In particular,
if communications experience nonnegligible delays,
the set $\calI_{i \leftarrow j}(t)$ will \emph{not} contain the most recent message $ I_{j}(t)$ and possibly other messages transmitted before time $t$.
This crucially affects control strategies based on distributed information exchange,
which must take delays into account to counterbalance them.
\revision{Note that,
	while we use the standard disk-based communication model~\cite{Capelli21icra-connectivityCBFDelays,Panagou16tac-multiRobotLyapunovBarrierFunction},
	we do not assume any model on communication delays and only assume that the available information at each robot is given by~\eqref{eq:hist}--\eqref{eq:info-ij-all}.}

\myParagraph{Safety of Networked Autonomous System}
We assume that each robot has a built-in, high-level autonomy that drives its plans and actions. 
This autonomy can be abstracted out as essentially implementing a nominal controller $\controllernom$ that computes action $\uref[t]{i}$ at time $t$ to fulfill a task of interest (\eg trajectory tracking).
However, this nominal control is not guaranteed to be \emph{safe} with respect to the network of robots 
(for example, it may cause two mobile robots to collide with each other).
We handle inter-robot interactions via a feedback controller $\controller$ that is in charge of safety and computes corrective control actions $\upert[t]{i}$ using feedback information received from other robots.
A general characterization is
\begin{equation}\label{eq:control-law}
	\upert[t]{i} = \controller\lr\xt[{[0,t]}]{i},\ut[[0,t)]{i},\lb\data{i}{j}(t)\rb_{j\in\robset\setminus\{i\}}\rr.
\end{equation}
Notation $\xt[{[0,t]}]{i}$ (resp., $\ut[[0,t)]{i}$) refers to all states (resp., control inputs) of robot $i$ till time $t$.
The safety-aware control input applied in~\eqref{eq:dynamics} is $\u[t]{i} = \uref[t]{i} + \upert[t]{i}$.
The feedback controller $\controller$ can use previous control actions $\ut[[0,t)]{i}$ \emph{excluding} $\u[t]{i}$ at time $t$,
because this is precisely computed from~\eqref{eq:control-law}.

We next propose a notion of safety for the networked autonomous system of robots.
The system
is \emph{safe} if the state $\x{}$ lies in the \emph{safe set} $\calS \subset \calX^\numRobots$. 
Importantly,
this involves \emph{all} robots in the system.
We make two assumptions on the structure of this safety set $\calS$ for a networked autonomous system,
which are instrumental for the theoretical result developed in~\autoref{sec:distributed-cbf}.
\begin{ass}[Safe set] 
	\label{ass:safe-set}
	There exists function $\psi$
	such that
	\begin{equation}\label{eq:safe-set}
		\calS = \left\{ \vxx \in\calX^\numRobots : \aggr_{j\in\robset}\psi(\vxx_i\ominus\vxx_j) \geq 0 \;\forall i\in\robset\right\},
	\end{equation}
	where $\aggr_{i\in\calV}$ is a permutation-invariant operation over set $\calV$
	and $\vy\ominus\vzz\in\calX$ is the difference between $\vy,\vzz\in\calX$.\footnote{
		In general,
		$\calX$ need not be Euclidean (\eg it can be a manifold).
		If $\calX$ is Euclidean,
		then $\ominus$ reduces to the standard difference between two vectors.
	}
\end{ass} 
\begin{ass}[Safety threshold]
	\label{ass:dist}
	There exists a continuous function $\rho:\Real{n} \rightarrow \reals$ such that
	$\rho(\vp_i - \vp_k) \ge 0$ implies
	$\sgn\aggr_{j\in\robset}\psi(\vxx_i\ominus\vxx_j) = \sgn\aggr_{j\in\robset\setminus\{k\}}\psi(\vxx_i\ominus\vxx_j) \ \forall k\neq i$.
\end{ass}
\cref{ass:safe-set} means that safety can be decomposed across robots.
That is,
the system is safe if and only if all robots are individually safe.
Moreover,
for each robot,
safety does not depend on other robots' \emph{identities} but only on their \emph{states} $\x{j}$'s.

\cref{ass:dist} states that safety of any robot is independent from those robots $k$ for which the condition $\rho(\vp_i - \vp_k) \ge 0$ holds.
This threshold condition has an impact on necessary and sufficient communication requirements.
In fact,
if the communication range between robots allows this condition to hold,
then safe control can be implemented based on local information exchange,
as will be clarified later in \autoref{sec:distr-cbf-perfect-info}.

The notion of safe set given by \cref{ass:safe-set,ass:dist} is general enough to encompass various safety applications,
\revision{as exemplified next.}
\begin{ex}[Collision avoidance]\label{example:collision-avoidance}
	\revision{We now show how model~\eqref{eq:safe-set} can be used to describe collision avoidance in a team of mobile robots $\robset$.}
	Let $d_\text{coll}$ denote the minimal distance two robots need to maintain to avoid collision. 
	Then, the safe set for collision avoidance is given by 
	\begin{equation}
		\calS = \left\{ \vxx \in \calX^\numRobots : \min_{j \in \robset\setminus\{i\}}\norm{\vp_i - \vp_j} \geq d_\text{coll} \; \forall i\in\robset\right\}.
	\end{equation}
	This is an instantiation of~\eqref{eq:safe-set} in Assumption~\ref{ass:safe-set} with $\psi(\x{i}\ominus\x{j}) \doteq \norm{\vp_i-\vp_j} - d_\text{coll}$ and the permutation-invariant operation $\min_{j \in \robset\setminus\{i\}}$.
	Assumption~\ref{ass:dist} is satisfied with $\rho(\vp_i - \vp_j) \doteq \norm{\vp_i - \vp_j} - d_\text{coll}$, 
	meaning that,
	if any two robots are enough apart, 
	they do not collide into each other.
\end{ex}

\subsection{Problem Statement}
\label{sec:problem-statement}

We aim to design a safety certification mechanism that attests whether the networked autonomous system is safe and design a minimally invasive corrective controller $\controller$, 
assuming a nominal controller $\controllernom$, 
that ensures that the state $\vxx$ remains safe. 

\begin{prob}[Distributed safety certification]\label{prob:distributed-safety-certificate}
	Find a distributed safety certification mechanism that can attest the safety of the networked autonomous system by using the local information at each robot $i$,
	\ie states $\xt[{[0,t]}]{i}$, 
	control actions $\ut[[0,t)]{i}$, 
	and communicated information $\{\data{i}{j}(t)\}_{j\in\robset\setminus\{i\}}$ at time $t$.
\end{prob}

\begin{prob}[Distributed safe controller]\label{prob:distributed-controller}
	Find a distributed minimally invasive controller $\controller$ that can ensure safety of the networked autonomous system when using the nominal controller $\controllernom$ by using the local information at each robot $i$,
	\ie states $\xt[{[0,t]}]{i}$, 
	control actions $\ut[[0,t)]{i}$, 
	and communicated information $\{\data{i}{j}(t)\}_{j\in\robset\setminus\{i\}}$ at time $t$.
\end{prob}

\begin{rem}
	While works in~\autoref{sec:cbf-safety} provides ways to design a safe minimally deviant controller, 
	these solutions are often either designed \textit{ad-hoc} or centralized.
	The latter case requires global information (the state $\x[t]{}$) to be known at a single unit or at each robot.
	Limited communication range implies that robots cannot access
	the full state in real time.
	Therefore, we need to design distributed safety mechanism and controller,
	while theoretically justifying when this is sufficient.
\end{rem}

\section{Certifiably Safe Distributed Control}
\label{sec:distributed-cbf}

In this and the following section,
we discuss the design of a distributed safety certification mechanism and a distributed controller that navigate the system to remain safe. 

In this section,
we first describe the design in the \emph{perfect information-exchange} case. 
With instantaneous communication,
we formally characterize when local information exchange between robots is theoretically sufficient to achieve safety under a distributed control strategy in \autoref{sec:distr-cbf-perfect-info}.
Drawing inspiration from this theoretical result,
we next propose a learning-based approach based on graph neural networks to design a distributed controller that implicitly learns to safely coordinate with other robots in \autoref{sec:distr-cbf-gnn}.

\subsection{Theoretical Foundations with Perfect Information Exchange}
\label{sec:distr-cbf-perfect-info}

We first show that,
for a safe set $\calS$ that satisfies Assumptions~\ref{ass:safe-set} and~\ref{ass:dist},
a distributed certificate ensures safety. 
Such a \emph{distributed control barrier function} uses only local information available at each robot.
Let the \emph{(safety) neighborhood} of robot $i$ be defined as $\neigh{i}\doteq\{j\in\robset\setminus\{i\}:\rho(\vp_i-\vp_j) \le 0\}$
and define the augmented vector $\xn{i}\in2^{\calX^R}$ of robot $i$ with $j$th element
\begin{equation}\label{eq:wi}
	\xn{ij} = \vxx_i \ominus \vxx_j \quad \forall \, j\in\neigh{i}.
\end{equation}
Note that $\xn{i}$ is defined whenever $\neigh{i} \neq \emptyset$.
In words,
$\xn{i}$ collects all relative state configurations between robot $i$ and its safety-relevant neighbors in the set $\neigh{i}$.
In fact,
by \cref{ass:dist},
all other robots outside $\neigh{i}$ are irrelevant for safety of robot $i$.
Vector $\xn{i}$ will be used by the distributed CBF to define a dynamic condition on the evolution of robots' states that keeps them safe.
We gather the control inputs of robots in the set $\neigh{i}\cup\{i\}$ in the vector $\un{i}$,
with $j$th element $\u{j}$ for all $j\in\neigh{i}\cup\{i\}$.
By the chain rule,
for a piece-wise differentiable function $h$,
it holds
\begin{equation}\label{eq:distributed-cbf-time-derivative}
	\cbfdot{\xn{i},\un{i}} = \sum_{j\in\neigh{i}} \nabla_{\xn{ij}}h(\xn{i})^\top \dot{\xn{}}_{ij}
\end{equation}
where $\dot{\xn{}}_{ij} = f\lr\x{i},\u{i}\rr - f\lr\x{j},\u{j}\rr$ from~\eqref{eq:wi} and~\eqref{eq:dynamics}.

\begin{definition}[Distributed CBF]\label{def:distributed-cbf}
	Let $\calM\subset2^{\calX^R}$ such that $\vxx\in\safe$ implies $\xn{i}\in\calM \; \forall i\in\robset$.
	A function $h:\calM \to \Real{}$ is a \emph{distributed control barrier function} for~\eqref{eq:dynamics},
	$i\in\robset$,
	if $h$ is continuously differentiable w.r.t. $\xn{i}$ for any fixed $\neigh{i}$ and there exists an extended class $\calK$ function $\alpha$ such that:
	\begin{enumerate}
		\item $h(\xn{i})\ge0$ if and only if $\aggr_{j \in \neigh{i}}\psi(\vxx_i\ominus\vxx_j) \ge 0$;
		\item $\sup_{\un{}}\dot{h}(\xn{i},\un{}) \ge -\alpha(h(\xn{i})) \ \forall i\in\robset, \xn{i}\in\calM$.
	\end{enumerate}
	In this case,
	the \emph{safe control set} associated with $\neigh{i} \cup \{i\}$ is
	\begin{equation}\label{eq:control-distributed-cbf}
		\calU_\safe\lr\xn{i}\rr \doteq \lb\un{}\in\calU^{|\neigh{i}|+1}: \dot{h}(\xn{i},\un{}) \ge -\alpha(h(\xn{i}))\rb.
	\end{equation}
\end{definition}
\begin{thm}[Distributed safety certification]\label{thm:distr-cbf}
	The safe set~\eqref{eq:safe-set} is equivalent to
	\begin{equation}\label{eq:safe-set-local}
		\calS = \lb \vxx \in\calX^\numRobots : \aggr_{j \in \neigh{i}} \psi(\vxx_i\ominus\vxx_j) \geq 0 \; \forall i\in\robset\rb.
	\end{equation}
	Moreover,
	let $h$ be a distributed CBF for~\eqref{eq:dynamics} for all $i\in\robset$.
	Then,
	any Lipschitz distributed controller $\pi:2^{\calX^R}\to \uset$ s.t. $\pi(\xn{i}) \in\calU_\safe(\xn{i})$ makes set $\safe$ forward invariant w.r.t.~\eqref{eq:dynamics}.
\end{thm}
\begin{proof}
	See \cref{app:proof}.
\end{proof}

In words,
\cref{def:distributed-cbf} means that safety is inherently \emph{local}.
While the networked autonomous system is safe whenever all robots are simultaneously safe according to~\eqref{eq:safe-set},
our characterization of distributed CBF in \cref{def:distributed-cbf} and \cref{ass:dist} imply that each robot $i$ is endowed with its individual safety that depends only on its neighbors.
In particular,
the time derivative $\dot{h}(\xn{i},\un{i})$ in condition~\eqref{eq:control-distributed-cbf} depends only on states and actions of robot $i$ and its neighbors in $\neigh{i}$.
This is key to enable safe distributed control,
whereby coordinating with neighbors is sufficient to formally ensure \emph{local safety} of each robot and in turn \emph{global safety} of the networked system according to \cref{thm:distr-cbf}.
On the other hand,
if at least one robot needed global system knowledge to preserve safety,
a distributed control strategy based on local communication may not work.

Given a distributed CBF $h$,
the constraint~\eqref{eq:control-distributed-cbf} can theoretically be used within a distributed optimization algorithm to compute safe control actions $\u{i}$ for all robots $i$.
However,
this approach raises two practical issues.
On the one hand,
the resulting distributed optimization problem need not be convex,
and convergence is not guaranteed.
On the other hand,
communication delays may disrupt applicability of this approach for online operation since several communication rounds and are typically needed to converge to a feasible solution.
These limitations motivate us to explore a learning-based approach in the next section,
whereby each robot collects information received from nearby robots but does not further communicate when computing the control actions.

\begin{rem}[Communication and safety]
	In real applications,
	robot $i$ needs to communicate at least with all robots in $\neigh{i}$ to implement a distributed controller based on~\cref{thm:distr-cbf}.
	Formally,
	this requires $\rho(\vp_{i}-\vp_j)\le0 \implies \norm{\vp_i-\vp_j}\le R_c$.
	In~\cref{example:collision-avoidance} this is equivalent to $R_c\ge d_\text{coll}$,
	which simply means that robots must communicate with neighbors (at least) before they collide to remain safe.
\end{rem}

\begin{figure*}[t]
	\centering
	\input{figures/architecture-diagram.tikz}
	\caption{%
		\revision{Proposed distributed controller implemented on robot $i$. 
		Robot $i$ receives information from nearby robots $\ell$, $j$, and $k$ and computes minimally invasive control actions $\u[t]{i}$. 
		Communication delays $\delta(t)$ are compensated by the predictor $\predmodel$ (red block).
		The GNN-based controller $\controlmodel$ (green block) computes corrective actions $\upert[t]{i}$ to ensure safety.}
	}
	\label{fig:architecture-diagram}
\end{figure*}

\subsection{Learning a Safe Distributed Controller}
\label{sec:distr-cbf-gnn}

\Cref{thm:distr-cbf} provides formal ground for the design of a distributed safety certification mechanism and controller.
We identify two main benefits of a learning-based design.
Firstly,
it circumvents the need to find a CBF via optimization-based parametric methods,
which can be computationally intractable.
Secondly,
it can allow robots to implicitly coordinate their actions by just processing the received information in $\data{i}{j}$,
without resorting to an iterative distributed algorithm.

\cref{ass:safe-set} and~\cref{def:distributed-cbf} imply that the distributed CBF is \emph{permutation-invariant},
meaning that any two neighbors of robot $i$ can be swapped without changing the value $\cbfval{\xn{i}}$.
Moreover,
recall that we assume all robots share the same dynamics and state and control sets.
Therefore,
embedding the information exchanged between robots into edge features,
we can model the distributed CBF $h$ as a GNN $\cbfmodel$ that operates on the graph $\calG(t) = (\robset, \calE(t))$,
where the directed edge $(i,j)\in\calE(t)$ if $\data{i}{j}(t)$ is nonempty,
\ie if robot $i$ has knowledge of robot $j$ at time $t$.
Under perfect information exchange,
this corresponds to robots $i$ and $j$ being within communication range at time $t$,
and hence we set $\data{i}{j}(t) = \mess{j}{}{t} = \{\x[t]{j}\}$ if $\norm{\vp_{i}(t)-\vp_j(t)}\le R_c$,
as this fully describes the current state of neighbor $j$,
and $\data{i}{j}(t) = \emptyset$ otherwise.
With this choice,
the edge feature $\featedge{i}{j}(t)$ from node (robot) $j$ to node $i$ in the GNN is the relative state $\xn[t]{ij}$,
which robot $i$ can compute online after receiving message $\mess{j}{}{t}$ from robot $j$ at time $t$.
We model the corrective controller $\controller$ as a second graph neural network $\controlmodel$ operating on graph $\calG(t)$ that computes corrective actions as 
\begin{equation}\label{eq:controller-gnn}
	\upert[t]{i} = \begin{cases}
		\control{\xn[t]{i}},	& \neigh{i}\neq\emptyset\\
		0,						& \mbox{otherwise}.						
	\end{cases}
\end{equation}
In the following,
we refer to~\eqref{eq:controller-gnn} as (distributed) GNN-based controller.
We jointly train the models $\cbfmodel$ and $\controlmodel$ using the following loss,
averaged across all robots in $\robset$,
\begin{equation}
	\label{eq:loss-cbf-controller}
	\calL(\paramcbf, \paramcontr) = \frac{1}{R}\sum_{i \in \robset}\lr \losscbf[i](\paramcbf,\paramcontr) + \losscontr[i](\paramcontr)\rr.
\end{equation}
The loss term related to the distributed CBF conditions is
\begin{multline}\label{eq:loss-cbf}
	\losscbf[i](\paramcbf,\paramcontr) \doteq w_\safe\sum_{\x{i}\sim\safe} \pos{-\cbfval[\paramcbf]{\xn{i}}+\epsilon} \\
	+ w_{\unsafe}\sum_{\x{i}\sim\unsafe} \pos{\cbfval[\paramcbf]{\xn{i}}+\epsilon}\\
	+ w_\text{der}\sum_{\x{i}\sim\calX} \pos{-\dot{\cbfmodel}(\xn{i},\un{i}) - \alpha(\cbfval[\paramcbf]{\xn{i}}) + \epsilon},
\end{multline}
where $\un{i}$ in the argument of the time derivative $\dot{\cbfmodel}$ gathers the control inputs of robots in the neighborhood $\neigh{i}\cup\{i\}$,
and each of these is computed as
\begin{equation}
	\u{j} = \control{\xn{j}}+\uref{j} \qquad \forall j\in\neigh{i}\cup\{i\},
\end{equation}
The loss term that forces minimally invasive corrective control is
\begin{equation}\label{eq:loss-control}
	\losscontr[i](\paramcontr) \doteq w_\text{contr} \sum_{\x{i} \sim \calX} \norm{\control{\xn{i}}}.
\end{equation}
For notation simplicity,
we write $\x{i}\in\safe$ (resp. $\x{i}\in\unsafe$) in~\eqref{eq:loss-cbf} if robot $i$ is safe (resp. unsafe).\footnote{
	The safe set is defined for (the neighborhood of) each robot $i$ according to \cref{ass:safe-set} and~\eqref{eq:safe-set-local}.
}
The coefficients $w_\safe$,
$w_{\unsafe}$,
$w_\text{der}$,
and $w_\text{contr}$ are training hyperparameters.
The first two terms (summations) in~\eqref{eq:loss-cbf} correspond to states sampled respectively from the safe set and the unsafe set,
whereas the third term enforces the dynamic condition of the distributed CBF.
While this is theoretically required only for states in the safe set,
we apply it to all training samples for robustness.
During training,
the time derivative $\dot{\cbfmodel}$ could be analytically computed by separately evaluating the gradient $\nabla_{\xn{i}}\cbfmodel(\xn{i})$ w.r.t. input features $\xn{i}$ and the time derivative $\dot{\vw}_i$ using~\eqref{eq:dynamics};
see~\eqref{eq:distributed-cbf-time-derivative}.
Alternatively,
a computationally cheaper way is to numerically approximate $\dot{\cbfmodel}$ by a finite difference,
\eg
\begin{equation}\label{eq:cbf-derivative}
	\cbfdot[\paramcbf]{\xn[t]{i},\un{i}(t)} \approx \dfrac{\cbfmodel(\xn[t+\step]{i})-\cbfmodel(\xn[t]{i})}{T_\text{s}},
\end{equation}
where $\step$ is a discrete time step and the state $\x[t+T_\text{s}]{i}$ of robot $i$ is generated from $\x[t]{i}$ through (discretized) dynamics~\eqref{eq:dynamics} applying the control action $\u[t]{i}$ computed with the learned controller as $\u[t]{i} = \uref[t]{i} + \control{\xn[t]{i}}$.
This model-free approach can also improve robustness to model uncertainty in~\eqref{eq:dynamics}.
Details on the learning architecture used in our experiments are provided in \cref{app:exp-models}.

\begin{rem}[Training samples]\label{rem:training}
	Instead of randomly sampling the state space,
	training samples can be drawn from simulated state trajectories.
	This helps to generate especially unsafe states that could occur during operation (\eg collisions) more easily than random samples.
	Moreover,
	sampling from trajectories allows one to approximate the time derivative of the distributed CBF directly from data,
	for example using~\eqref{eq:cbf-derivative}.
\end{rem}

\section{Safe Distributed Control with Realistic Information Exchange}
\label{sec:distr-cbf-realistic-information}

In the realistic information-exchange case, 
robot $i$ does not have instantaneous access to neighbors' states because messages transmitted by other robots are delayed.
This means that the distributed CBF and controller cannot use the actual value of $\xn[t]{i}$ to compute $\u[t]{i}$.
To tackle this,
we design a \emph{predictor} that estimates $\xn[t]{i}$ based on information received till time $t$.

To counterbalance communication delays,
robots need to transmit more information compared to the perfect information-exchange case.
Considering time-slotted communication with time index $t$,
robot $j$ sends the following message at time $t$:
\begin{equation}
	\label{eq:info-ij-att}
	\mess{j}{}{t} = \lb\x[t]{j}, \vu_j(t-1), t\rb,
\end{equation}
where $\u[t-1]{j}$ is the latest control input computed before time $t$.
Note that robot $j$ \emph{cannot} transmit $\u[t]{j}$ at time $t$ since this is \emph{computed} at time $t$.
%
We let robots transmit previous control actions with the intuition that this helps the neighbors to ``reconstruct'' its trajectory,
enhancing prediction quality.
For a message $I_j(t')$ that has been received by robot $i$ at time $t$,
robot $i$ constructs the corresponding \textit{relative message} as
\begin{equation}\label{eq:info-relative}
	\mess{j}{i}{t,t'} \doteq \lb  \x[t']{i} \ominus \x[t']{j}, \u[t'-1]{i} \ominus \u[t'-1]{j}, t-t' \rb
\end{equation}
and the \textit{relative dataset} at robot $i$ referred to robot $j$ is
\begin{equation}\label{eq:dataset-relative}
	\datarel{i}{j}(t) \doteq \lb \mess{j}{i}{t,t'}: t'\in\calT_{i \leftarrow j}(t)\rb.
\end{equation}
The time lag $t-t'$ in~\eqref{eq:info-relative} is the \emph{Age-of-Information} that represents staleness of data in $\mess{j}{i}{t,t'}$.
The predictor $\lambda$ of robot $i$ maps received data as
\begin{equation}\label{eq:prediction-map}
	\lambda: \datarel{i}{j}(t) \longmapsto\predval{ij}{t},
\end{equation}
where the output $\predval{ij}{t}$ is intended to be an estimate of $\xn[t]{ij}$.
In words,
robot $i$ infers the current relative state of robot $j$ by ``projecting'' past states and actions to the current time.
Note that this mapping makes sense only if the relative dataset $\datarel{i}{j}(t)$ is not empty,
that is,
if $j$ is considered a neighbor at time $t$.
To reduce storage requirements and bound the message-passing graph,
robots discard all relative messages with age greater than a threshold $\aoimax$,
so that $j$ is removed from the neighbors of $i$ when no data with age at most $\aoimax$ are received.
The output of the predictor is fed to the distributed CBF and controller in place of the unknown $\xn[t]{ij}$.
Specifically,
the predictor-based controller computes corrective actions as
\begin{equation}\label{eq:controller-predictor}
	\upert[t]{i} = \begin{cases}
		 \controlmodel(\predval{i}{t}),	& \exists j : \datarel{i}{j}(t)\neq\emptyset\\
		 0, 						& \mbox{otherwise}.
	\end{cases}
\end{equation}
We refer to~\eqref{eq:controller-predictor} as predictor-based (distributed) control.
The overall predictor-controller architecture for realistic-information exchange is schematized in~\autoref{fig:architecture-diagram}.
\revision{The dashed lines show that the CBF model $\cbfmodel$ (blue block) is used in training to enforce that the learned corrective control actions $\upert[t]{i}$ satisfy the safety conditions,
	but it is not deployed during operation.}

\subsection{Learning the Predictor}
\label{sec:learn-predictor}

In view of the potentially complex dynamics~\eqref{eq:dynamics} with the nonlinear controller $\controlmodel$,
we parametrize the predictor as a neural network $\predmodel$.
This also allows us to integrate the predictor learning into the same training procedure used for the GNN-based distributed CBF and controller.
We train the predictor with the following loss that penalizes the relative mismatch from the output $\predval{ij}{t}$ to the current value $\xn[t]{ij}$:
\begin{equation}\label{eq:loss-pred}
	\losspred(\parampred) = \sum_{i,j \in \robset:\datarel{i}{j}(t)\neq\emptyset}
	\frac{\norm[\calX]{\predval{ij}{t} \ominus \xn[t]{ij}}}{\norm[\calX]{\xn[t]{ij}}}.
\end{equation}
When training both predictor model and distributed CBF and controller models in the same training session,
we replace all values computed with perfect information,
namely $\cbfval{\xn[t]{i}}$ and $\control{\xn[t]{i}}$,
with values computed from predicted information,
namely $\cbfval{\predval{i}{t}}$ and $\control{\predval{i}{t}}$,
in the loss terms $\losscbf[i]$ and $\losscontr[i]$ in~\eqref{eq:loss-cbf-controller} .
However,
we have experimentally noted that learning accuracy improves when computing the time derivative $\dot{\cbfmodel}$ of the distributed CBF with actual values of $\xn[t]{i}$ rather than with predicted values $\predval{i}{t}$.
Importantly,
this can be done because the derivative $\dot{\cbfmodel}$ is explicitly computed only when training the distributed CBF model $\cbfmodel$ offline,
but it is not needed during online operation,
when only the controller $\controlmodel$ is deployed.
Moreover,
we have experimentally observed that alternatively training (updating) the predictor and,
jointly,
CBF and controller favors learning stability.
This happens because the predictor is trained under a consistent control behavior,
while CBF and controller learn corrective actions for given prediction capabilities.
Conversely,
jointly training all three models resulted in learning failure.

Under realistic information-exchange,
it makes sense to draw training samples from simulated trajectories rather than randomly from the state set $\calX$.
In fact,
this allows us to train both predictor and distributed CBF and controller with states and control actions that are consistent with data received under delays.
In contrast,
randomly sampling states $\x[t]{i}$ from $\calX$ would also require to randomly generate relative messages $\mess{j}{i}{t,t'}$ that need not be consistent with feasible trajectories.

\subsection{Heuristic Control to Enhance Performance}\label{sec:heuristic}

Learning-based controllers can hardly learn zero actions $\upert[t]{i}$,
which causes continuous perturbations of the nominal input $\uref[t]{i}$ if neighbors are present.
This means that applying the learning-based controller may make the robots deviate from the desired behavior even when no corrective action is required.
This effect may be counteracted by implementing task-specific adjustments.
For example,
in the case of point-stabilization or goal-reaching,
robots may switch to the nominal controllers as soon as they are sufficiently close to the goal. 
While such \textit{ad-hoc} strategies can be implemented with system-specific knowledge,
in this section we explore a task-agnostic strategy that draws inspiration from the distributed CBF condition.
The key intuition is that the higher the value of the CBF,
the ``safer'' a robot is.
Hence,
CBF values may be thresholded to assess if corrective actions are actually needed.
We propose the following switching control that corrects nominal actions according to the condition on the distributed CBF dynamics:
\begin{equation}\label{eq:control-heuristic}
	\uperth[t]{i} = \begin{cases}
		0							& \cbfdot{\predval{i}{t},\control{\predval{i}{t}}}+\alpha\cbfval{\predval{i}{t}}\ge\varphi\epsilon\\
		\upert[t]{i}	& \mbox{otherwise}.
	\end{cases}
\end{equation}
The parameter $\epsilon$ in~\eqref{eq:control-heuristic} is the same used in training and
the tunable coefficient $\varphi\ge0$ favors safe (large $\varphi$) or aggressive behavior (small $\varphi$).
In words,
control~\eqref{eq:control-heuristic} uses the nominal input if the next state is deemed safe by the learned CBF and applies corrective actions otherwise.
Crucially,
the time derivative $\cbfdot{\predval{i}{t},\un{i}(t)}$ cannot be computed in a distributed fashion during operation.
To implement~\eqref{eq:control-heuristic} in a distributed way,
robots compute $\predval{ij}{t+1}$ with the predictor given available information $\datarel{i}{j}(t)$,
and locally estimate the time derivative as
\begin{equation}\label{eq:cbf-derivative-est}
	\cbfdot{\predval{i}{t},\control{\predval{i}{t}}} \approx \dfrac{\cbfval{\predval{i}{t+1}} - \cbfval{\predval{i}{t}}}{T_\text{s}}.
\end{equation}

\begin{rem}[Heuristic control]
	This control strategy is purely heuristic and need not retain any safety guarantee of the distributed CBF-based learned controller.
	We use~\eqref{eq:control-heuristic}--\eqref{eq:cbf-derivative-est} in our experiments just to compare it to the GNN-based controller,
	to test if we can improve performance sacrificing a little safety.
\end{rem}
	

\section{Experiments}
\label{sec:experiments}

We test our proposed controller for multi-robot navigation,
where the robots have to reach goals while avoiding collisions.\footnote{
	The code for training and simulations is available at \url{https://github.com/lucaballotta/macbf-gnn}.
}

\subsection{Experiment Setup}
\label{sec:exp-setup}

\begin{figure}
	\begin{framed}
		\centering
		\includegraphics[width=.8\linewidth]{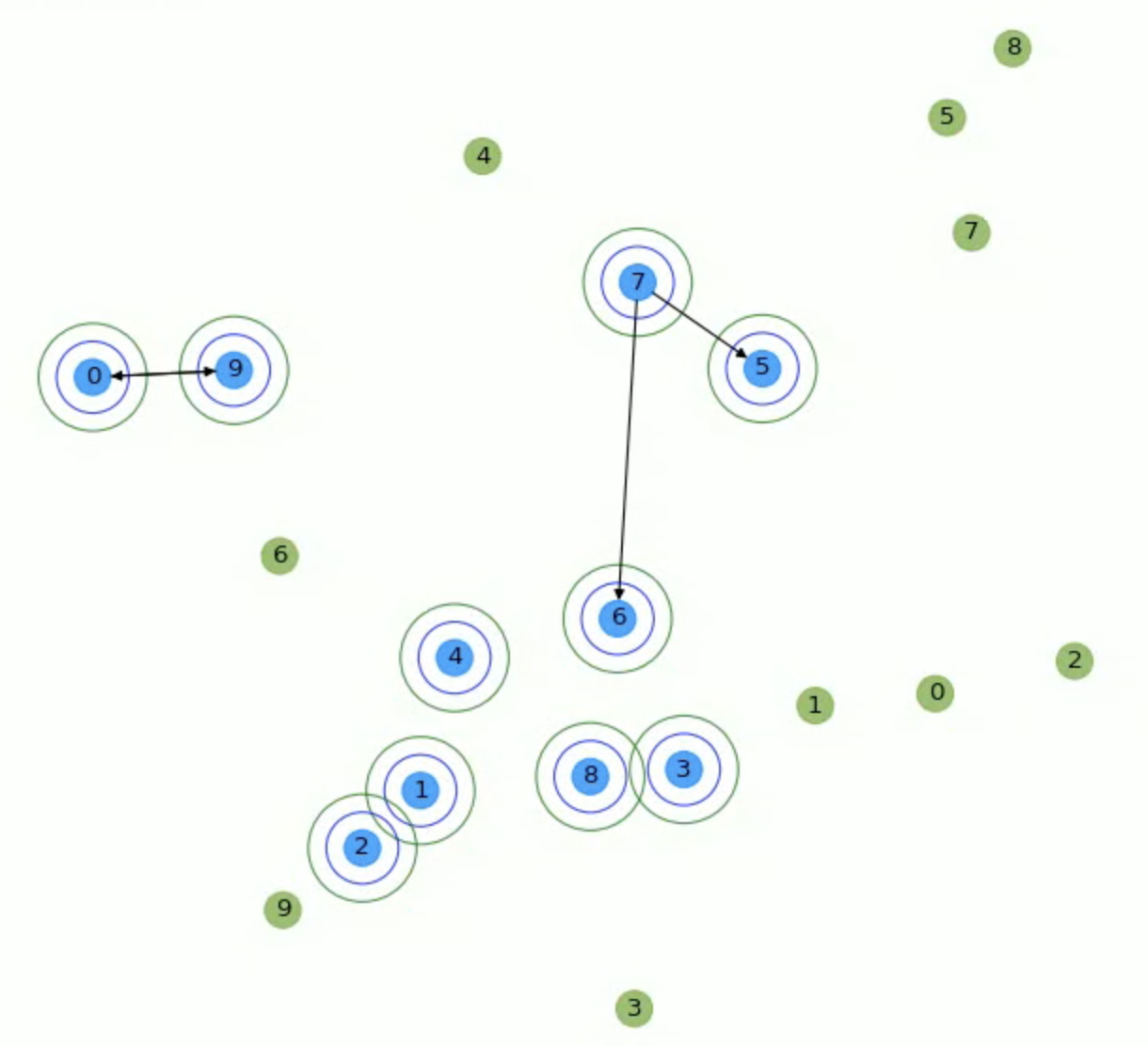}
	\end{framed}
	\caption{\revision{Snapshot of simulation.
		Robots (blue) have to reach goals (green) without colliding.
		The inner circles around robots marks the closest distance before collision.
		The outer circles determine the safe set,
		namely robots outside them are tagged safe in training.}
	}
	\label{fig:exp}
\end{figure}

\myParagraph{Control task}
We implement and study our proposed approach for a multi-robot system concerned with collision avoidance as in \cref{example:collision-avoidance}.
The simulation snapshot in \autoref{fig:exp} illustrates the experimental setup.
We simulate mobile robots as 2D balls with radius $0.05$ that move on the planar workspace $[0, 3] \times [0,3] \subset\Real{2}$.
Each robot (\revision{blue circles in \autoref{fig:exp}}) has to reach a pre-specified goal location (\revision{green circles}) starting from an initial location,
all different across robots.
To fulfill this task,
each robot is equipped with a nominal controller $\controllernom$ that drives it from its current location to its goal through a straight line to optimize control effort,
which may result in collisions.
Two robots collide if the distance between their centers drops below $d_\text{coll}=0.1$ (\revision{innermost ring around each robot in \autoref{fig:exp}}),
whereas we tag safe robot states when their distance is at least $0.2$ (\revision{outermost ring}).
The communication radius is $R_c=1$.
The arrows connecting pairs of robots in \autoref{fig:exp} show the edges $\calE$ of the message passing graph $\calG$ used by the distributed GNN-based controller.
Several edges are missing because the snapshot is taken at the beginning of an episode ($t=0$) and few messages $\mess{j}{}{0}$ have been received because of communication delays,
hence many message datasets $\data{i}{j}$ are still empty.

\myParagraph{Robot dynamics}
We consider two dynamical models.

\begin{description}
	\item[Single integrator:] The state of robot $i$ evolves as
	\begin{equation}
		\label{eq:agent-dynamics-single-integrator}
		\xdot[t]{i} = \u[t]{i}
	\end{equation}
	with $\x[t]{i} = \vp_i(t) \in \Real{2}$.
	Even though this model is simple, 
	it reliably represents robot motion provided that low-level controllers track the reference velocity fast enough compared to position tracking~\cite{Cavorsi22rss-resilienceCBF,Capelli21icra-connectivityCBFDelays,Lee13tm-teleoperation,Rossi23automatica-sampledCommunication}.
	We set a speed limit on each direction as $\norm[1]{\u{i}}\le 0.4$.
	\item[Dubins car:] The state of robot $i$ is a $4$-dimensional vector that stacks respectively position $\vp_{i}\in\Real{2}$,
	speed $v_i\in\Real{}$,
	and heading angle $\vartheta_i\in\Real{}$.
	Speed and heading angle are directly actuated.
	The dynamics are
	\begin{equation}
		\label{eq:agent-dynamics-dubins-car}
		\begin{bmatrix}
			\multirow{2}{*}{$\dot{\vp}_{i}(t)$}\\
			\\
			\dot{v}_i(t)\\
			\dot{\vartheta}_i(t)
		\end{bmatrix} = 
		\begin{bmatrix}
			v_i(t)\cos(\vartheta_i(t))\\
			v_i(t)\sin(\vartheta_i(t))\\
			\multirow{2}{*}{$\u[t]{i}$}\\
			\\
		\end{bmatrix}.
	\end{equation}
	Robots are subject to acceleration limit $|u_{i,1}| \le 10$ and angular speed limit $|u_{i,2}| \le 100$.
\end{description}

At each simulation step,
we discretize the dynamics via forward Euler with time step $\step = 0.03$.

\myParagraph{Communication model}
We model communication delays as Poisson random variables with expectation that is proportional to the number of neighbors within communication range.
For each robot $i$,
transmitted message $I_i(t)$ incurs $\delays[i]{t}$ delay steps,
where
\begin{equation}\label{eq:comm-delays-sim}
	\delays[i]{t} \sim \mathrm{Pois}(c_\text{del}|\neigh[t]{i}|),
\end{equation}
$\neigh[t]{i} \doteq \left\{ j\in\robset : \norm{\vp_i(t) - \vp_j(t)} \leq R_c\right\}$ gathers all robots within communication range at time $t$,
and the coefficient $c_\text{del}$ depends on channel conditions and transmission power.
Note that larger $c_\text{del}$ or more neighbors induce longer delays.
While we \emph{generate} delays according to~\eqref{eq:comm-delays-sim},
this model is not used either to train or to test the models.
Moreover,
note that delays are in general asymmetric and vary overtime.

\myParagraph{Training}
We first train controllers under perfect information exchange according to \autoref{sec:distr-cbf-gnn},
by manually setting to zero all transmission delays.
We then train pairs of predictor and controller models as discussed in \autoref{sec:distr-cbf-realistic-information}.
For the single integrator model we use two delay coefficients $c_\text{del}\in\{0.5,0.8\}$,
whereas for the Dubins car model we set $c_\text{del} = 1.5$.
To generate training samples,
we simulate trajectories of $\numRobots=10$ robots by running multiple training episodes.
We randomly assign start and goal locations to all robots at the beginning of each episode.
Training episodes end either when the distance between all robots and their respective goals is at most $0.02$ or if a pre-defined deadline of $500$ simulation steps expires.
At each simulation step,
we store all robot states $\x{}$ and received messages $\data{i}{j}$ in a buffer,
from which we periodically draw a batch of training samples to update the models.
The loss used for each update is the average loss across samples of the batch.
To avoid imbalance in training the CBF,
we sample an equal number of safe and unsafe states or,
when this is not possible,
we draw all unsafe states (which are usually fewer than safe states) from the buffer.
We update the models every $\Delta_\text{train} = 512$ simulation steps for $n_\text{desc} = 10$ descent steps (backpropagation),
For the single integrator model,
we update predictor or CBF/controller $ 10$ consecutive times before switching to the other model(s).
For the Dubins car model,
we first train only the controller under perfect information exchange,
and then train only the predictor under communication delays.
We use an $\varepsilon$-greedy on-policy scheme,
\ie control actions applied during training are computed with the trained controller with probability $1-\varepsilon$
and with only the nominal controller $\pi_n$ with probability $\varepsilon$.
We set a decreasing probability $\varepsilon_t = \nicefrac{1}{t}$ to favor exploration during early iterations.
In the perfect information-exchange case,
we train CBF and controller models for $500000$ simulation steps.
In the realistic information-exchange case we train CBF, controller, and predictor for a million simulation steps with single integrator dynamics,
taking about $5$ hours on an NVIDIA RTX A4000,
while for the Dubins car model we train the predictor for $2.5$ million simulation steps.
Details about the learning architecture and neural network models are provided in \cref{app:exp-models}.

\subsection{Results and Discussion}
\label{sec:exp-results}

\begin{figure}
	\centering

\begin{tikzpicture}
	
	\begin{axis}[
		width=.8\linewidth,
		height=.45\linewidth,
		legend cell align={left},
		legend style={font=\small},
		legend pos=south west,
		tick align=inside,
		tick pos=left,
		x grid style={white!70!black},
		xlabel={Number of robots $R$},
		xmajorgrids,
		xtick style={color=black},
		y grid style={white!70!black},
		ylabel={$\srate$},
		ymajorgrids,
		ytick style={color=black},
		xtick={10, 12, 15, 18, 20},
		tick label style={font=\footnotesize},
		label style={font=\small}
		]
		\addplot [semithick, mark=square*, mark size=3, mark options={solid}, blue]
		table {%
			10 1
			12 1
			15 1
			18 1
			20 1
		};
		\addlegendentry{Single integrator}
		\addplot [semithick, mark=*, mark size=3, mark options={solid}]
		table {%
			10 1
			12 .95
			15 .92
			18 1
			20 .82
		};
		\addlegendentry{Dubins car}
	\end{axis}
	
\end{tikzpicture}
	\caption{\revision{Safety rates with controller~\eqref{eq:controller-gnn} in the perfect information-exchange case.}
	}
	\label{fig:safety}
\end{figure}

To evaluate safety,
we compute the \emph{safety rate} as the fraction of safe test episodes:
\begin{equation}
	\srate \doteq \dfrac{\#\mbox{ test episodes without collisions}}{\#\mbox{ test episodes}}.
\end{equation}
Index $\srate$ does not account for how many collisions happen within a single episode,
so it considers any episode unsafe even if a single collision occurs.
For all tests,
we average results over $100$ Monte Carlo runs.

We first address safety under perfect information exchange,
\ie instantaneous inter-robot communication,
and apply our distributed GNN-based controller~\eqref{eq:controller-gnn} jointly trained with the distributed control barrier function model as discussed in \autoref{sec:distr-cbf-gnn}.
We train all models with $R=10$ and test how safety rate varies as we increase the number of robots deployed during testing.
By doing so,
we actually increase \emph{density} of robots moving in the workspace since the latter has fixed size.
The results are shown in \autoref{fig:safety}.
The learned controller for single integrator model easily scales to double the robot density without decreasing its safety rate,
which is always $1$.
The Dubins' car model instead causes oscillations of the safety rate due to the more complex nonlinear behavior.
Nonetheless,
the safety rate remains over $80\%$ even when doubling robot density.

\begin{figure}
	\centering
	\subfloat[Single integrator model.\label{fig:safety_del_si}]{

\begin{tikzpicture}
	
	\begin{axis}[
		width=.8\linewidth,
		height=.4\linewidth,
		tick align=inside,
		tick pos=left,
		x grid style={white!70!black},
		xlabel={Communication delay coefficient $c_\text{del}$},
		xmajorgrids,
		xtick style={color=black},
		y grid style={white!70!black},
		ylabel={$\srate$},
		ymajorgrids,
		ytick style={color=black},
		ytick={0.4, 0.5, 0.6, 0.7, 0.8, 0.9, 1.0},
		xtick={0.2, 0.5, 0.8},
		tick label style={font=\footnotesize},
		label style={font=\small}
		]
		\addplot [semithick, mark=triangle*, mark size=3, mark options={solid}, red]
		table {%
			0.2 0.59
			0.5 0.45
			0.8 0.38
		};
		\addplot [semithick, mark=*, mark size=3, mark options={solid}]
		table {%
			0.2 0.99
			0.5 0.98
			0.8 0.93
		};
	\end{axis}
	
\end{tikzpicture}
	}
	\\
	\subfloat[Dubins car model.\label{fig:safety_del_dc}]{

\begin{tikzpicture}
	
	\begin{axis}[
		width=.8\linewidth,
		height=.5\linewidth,
		legend cell align={left},
		legend style={font=\small,fill opacity=.75},
		legend pos=south west,
		tick align=inside,
		tick pos=left,
		x grid style={white!70!black},
		xlabel={Communication delay coefficient $c_\text{del}$},
		xmajorgrids,
		xtick style={color=black},
		y grid style={white!70!black},
		ylabel={$\srate$},
		ymajorgrids,
		ytick style={color=black},
		ytick={0.4, 0.5, 0.6, 0.7, 0.8, 0.9, 1.0},
		xtick={0.5, 0.8, 1.0, 1.2},
		tick label style={font=\footnotesize},
		label style={font=\small}
		]
		\addplot [semithick, mark=triangle*, mark size=3, mark options={solid}, red]
		table {%
			0.5 0.82
			0.8 0.61
			1.0 0.58
			1.2 0.37
		};
		\addlegendentry{GNN controller}
		\addplot [semithick, mark=*, mark size=3, mark options={solid}]
		table {%
			0.5 0.89
			0.8 0.84
			1.0 0.83
			1.2 0.68
		};
		\addlegendentry{GNN controller + predictor}
	\end{axis}
	
\end{tikzpicture}
	}
	\caption{\revision{Safety rates with controller~\eqref{eq:controller-gnn} (``GNN controller'') and with predictor-based controller~\eqref{eq:controller-predictor} (``GNN controller + predictor'') in the realistic information-exchange case.
			Our proposed control architecture sizably improves safety under communication delays.}
	}
	\label{fig:safety_del}
\end{figure}

We then turn to the realistic information-exchange setting,
where we compare the two distributed control approaches with and without predictor.
Figure~\ref{fig:safety_del} conveys the main message.
Awareness of communication delays is crucial to effectively handle data received by nearby robots during operation.
The ``GNN controller'' is the distributed GNN-based controller~\eqref{eq:controller-gnn} trained under perfect-information exchange with no prediction.
The ``GNN controller + predictor'' is the predictor-based controller~\eqref{eq:controller-predictor} illustrated in~\autoref{fig:architecture-diagram} trained under realistic-information exchange,
where the predictor estimates the relative states $\xn{ij}$ from delayed information $\datarel{i}{j}$ tagged with AoI.
The plot shows that the predictor-based controller sizably increases the safety rate as opposed to the control~\eqref{eq:controller-gnn} that neglects communication delays.
This witnesses that the CBF condition critically depends on timeliness of information,
especially at the boundary between safe and unsafe regions (\ie when two robots are about to collide).
To test ``GNN controller + predictor'' trained under realistic information-exchange for single integrator dynamics (\cref{fig:safety_del_si}),
we use the same delay coefficient $c_\text{del}$ set in training,
except for the test with delay coefficient $c_\text{del} = 0.2$ where we use the controller trained with $c_\text{del} = 0.5$.
Conversely,
the predictor for the Dubins car model is trained under a large coefficient $c_\text{del} = 1.5$ and we study safety for smaller delays (\cref{fig:safety_del_dc}). 
The tests suggest that the controllers generalize to delay distributions with smaller delay expectations.
It is important to remark that we consider \emph{distributed controllers} because we directly deploy the trained GNN-based models $\pi_{\xi}$.

\begin{figure}
	\centering

\begin{tikzpicture}
	
	\begin{axis}[
		width=.8\linewidth,
		height=.5\linewidth,
		legend cell align={left},
		legend style={fill opacity=.75, draw opacity=1, text opacity=1, font=\small},
		legend columns=2,
		legend pos=north west,
		tick align=inside,
		tick pos=left,
		x grid style={white!70!black},
		xlabel={Communication delay coefficient $c_\text{del}$},
		label style={font=\small},
		tick label style={font=\footnotesize},
		xmajorgrids,
		xtick style={color=black},
		y grid style={white!70!black},
		ylabel={Distance from goal},
		ymajorgrids,
		ymax=0.2,
		ytick style={color=black},
		xtick={0.2, 0.5, 0.8},
		xticklabels={0.2, 0.5, 0.8}
		]
		\addplot [mark=*, mark size=3, mark options={solid}]
		table {%
			0.2 0.038
			0.5 0.071
			0.8 0.118
		};
		\addlegendentry{Realistic-info}
		\addplot[black, ultra thick]
		coordinates {(.2,.029)(.8,.029)};
		\addlegendentry{Perfect-info}
 		\addplot[blue, ultra thick, dashed]
 		coordinates {(.2,.02)(.8,.02)};
 		\addlegendentry{Target distance}
 		\addplot[red, ultra thick, dotted]
 		coordinates {(.2,.05)(.8,.05)};
 		\addlegendentry{Robot radius}
	\end{axis}
	
\end{tikzpicture}
	\caption{Average final distance from goal with trained GNN-based controller.}
	\label{fig:performance}
\end{figure}
\begin{figure}
	\centering
	\subfloat[Average trajectory length.
	\label{fig:performance-heuristics}
	]{

\begin{tikzpicture}
	
	\begin{axis}[
		width=.5\linewidth,
		height=.45\linewidth,
		tick align=inside,
		tick pos=left,
		legend style={at={(0.03,0.16)}, anchor=south west, fill opacity=.75, draw opacity=1, text opacity=1, font=\small},
		x grid style={white!70!black},
		xlabel={Delay coefficient $c_\text{del}$},
		xmajorgrids,
		xtick style={color=black},
		y grid style={white!70!black},
		ylabel={Trajectory length},
		label style={font=\small},
		tick label style={font=\footnotesize},
		ymajorgrids,
		ytick style={color=black},
		xtick={0.2, 0.5, 0.8},
		xticklabels={0.2, 0.5, 0.8}
		]
		\addplot [mark=*, mark size=3, mark options={solid}]
		table {%
			0.2 308.95
			0.5 313.58
			0.8 314.77
		};
		\addlegendentry{Realistic-info}
		\addplot [ultra thick, dashed] 
		coordinates {(.2,292.14) (.8,292.14)};
		\addlegendentry{Perfect-info}
		
	\end{axis}
	
\end{tikzpicture}}
	\subfloat[Safety rate.
	\label{fig:safety-heuristics}
	]{

\begin{tikzpicture}
	
	\begin{axis}[
		width=.5\linewidth,
		height=.45\linewidth,
		tick align=inside,
		tick pos=left,
		x grid style={white!70!black},
		xtick style={color=black},
		y grid style={white!70!black},
		ylabel={$\srate$},
		y grid style={white!70!black},
		ymajorgrids,
		xmajorgrids,
		xlabel={Delay coefficient $c_\text{del}$},
		label style={font=\small},
		tick label style={font=\footnotesize},
		xtick={0.2, 0.5, 0.8},
		xticklabels={0.2, 0.5, 0.8}
		]
		\addplot [red, mark=triangle*, mark size=3.5]
		table {%
			0.2 0.99
			0.5 0.97
			0.8 0.81
		};
	\end{axis}
	
\end{tikzpicture}}
	\caption{Experiments with heuristic control~\eqref{eq:control-heuristic} and single integrator model.}
	\label{fig:heuristic}
\end{figure}

\begin{figure*}[ht]
	\centering
	\subfloat[Coefficient $c_\text{del}=0.2$.
	\label{fig:safety_scalability_0,2}]{

\begin{tikzpicture}[scale=.47]
	
	\begin{axis}[
		legend cell align={left},
		legend style={fill opacity=.75, draw opacity=1, text opacity=1, font=\LARGE},
		legend pos=south west,
		tick align=inside,
		tick pos=left,
		x grid style={white!70!black},
		xlabel={Number of robots $R$},
		xmajorgrids,
		xtick style={color=black},
		y grid style={white!70!black},
		ylabel={$\srate$},
		ymajorgrids,
		ytick style={color=black},
		tick label style={font=\Large},
		label style={font=\LARGE}
		]
		\addplot [semithick, mark=*, mark size=5, mark options={solid}, red]
		table {%
			10 .99
			15 1.
			18 .97
			20 .99
		};
		\addlegendentry{GNN controller}
		\addplot [semithick, mark=triangle*, mark size=5, mark options={solid}]
		table {%
			10 .99
			15 .94
			18 .89
			20 .76
		};
		\addlegendentry{Heuristic}
	\end{axis}
	
\end{tikzpicture}
	}
	\hfil
	\subfloat[Coefficient $c_\text{del}=0.5$.
	\label{fig:safety_scalability_0,5_arx}]{

\begin{tikzpicture}[scale=0.47]
	
	\begin{axis}[
		legend cell align={left},
		legend style={fill opacity=.75, draw opacity=1, text opacity=1, font=\LARGE},
		legend pos=south west,
		tick align=inside,
		tick pos=left,
		x grid style={white!70!black},
		xlabel={Number of robots $R$},
		xmajorgrids,
		xtick style={color=black},
		y grid style={white!70!black},
		ylabel={$\srate$},
		ymajorgrids,
		ytick style={color=black},
		ytick={0.2, 0.3, 0.4, 0.5, 0.6, 0.7, 0.8, 0.9, 1},
		xtick={10, 13, 15, 18},
		tick label style={font=\Large},
		label style={font=\LARGE}
		]
		\addplot [semithick, mark=*, mark size=5, mark options={solid}, red]
		table {%
			10 .98
			13 .96
			15 .84
			18 .45
		};
		\addlegendentry{GNN controller}
		\addplot [semithick, mark=triangle*, mark size=5, mark options={solid}]
		table {%
			10 .97
			13 .82
			15 .77
			18 .24
		};
		\addlegendentry{Heuristic}
	\end{axis}
	
\end{tikzpicture}
	}
	\hfil
	\subfloat[Coefficient $c_\text{del}=0.8$.
	\label{fig:safety_scalability_0,8_arx}]{

\begin{tikzpicture}[scale=0.47]
	
	\begin{axis}[
		legend cell align={left},
		legend style={fill opacity=.75, draw opacity=1, text opacity=1, font=\LARGE},
		legend pos=south west,
		tick align=inside,
		tick pos=left,
		x grid style={white!70!black},
		xlabel={Number of robots $R$},
		xmajorgrids,
		xtick style={color=black},
		y grid style={white!70!black},
		ylabel={$\srate$},
		ymajorgrids,
		ytick style={color=black},
		ytick={0.2, 0.3, 0.4, 0.5, 0.6, 0.7, 0.8, 0.9, 1},
		ymax=1.05,
		ymin=.15,
		xtick={10, 12, 15},
		tick label style={font=\Large},
		label style={font=\LARGE}
		]
		\addplot [semithick, mark=*, mark size=5, mark options={solid}, red]
		table {%
			10 .93
			12 .71
			15 .31
		};
		\addlegendentry{GNN controller}
		\addplot [semithick, mark=triangle*, mark size=5, mark options={solid}]
		table {%
			10 .79
			12 .54
			15 .29
		};
		\addlegendentry{Heuristic}
	\end{axis}
	
\end{tikzpicture}
	}
	\caption{Safety rate when increasing robot density in test ($R=10$ in training).}
	\label{fig:safety_scalability}
\end{figure*}
\begin{figure*}
	\centering
	\subfloat[Coefficient $c_\text{del}=0.2$.
	\label{fig:performance_scalability_0,2}]{

\begin{tikzpicture}[scale=0.47]
	
	\begin{axis}[
		tick align=inside,
		tick pos=left,
		x grid style={white!70!black},
		xlabel={Number of robots $R$},
		xmajorgrids,
		xtick style={color=black},
		y grid style={white!70!black},
		ylabel={Trajectory length},
		label style={font=\LARGE},
		tick label style={font=\Large},
		ymajorgrids,
		ytick style={color=black},
		ymin=315,
		ymax=425,
		xtick={10, 15, 18, 20}
		]
		\addplot [mark=*, mark size=5, mark options={solid}]
		table {%
			10 332.65
			15 391.25
			18 408.89
			20 418.33
		};
	\end{axis}
	
\end{tikzpicture}
	}
	\hfil
	\subfloat[Coefficient $c_\text{del}=0.5$.
	\label{fig:performance_scalability_0,5}]{

\begin{tikzpicture}[scale=0.47]
	
	\begin{axis}[
		tick align=inside,
		tick pos=left,
		x grid style={white!70!black},
		xlabel={Number of robots $R$},
		xmajorgrids,
		xtick style={color=black},
		y grid style={white!70!black},
		ylabel={Trajectory length},
		label style={font=\LARGE},
		tick label style={font=\Large},
		ymajorgrids,
		ytick style={color=black},
		ymin=325,
		ymax=395,
		ytick distance=10,
		xtick={10, 13, 15, 18}
		]
		\addplot [mark=*, mark size=5, mark options={solid}]
		table {%
			10 338.86
			13 343.96
			15 363.43
			18 384.66
		};
	\end{axis}
	
\end{tikzpicture}
	}
	\hfil
	\subfloat[Coefficient $c_\text{del}=0.8$.
	\label{fig:performance_scalability_0,8}]{

\begin{tikzpicture}[scale=0.47]
	
	\begin{axis}[
		tick align=inside,
		tick pos=left,
		x grid style={white!70!black},
		xlabel={Number of robots $R$},
		xmajorgrids,
		xtick style={color=black},
		y grid style={white!70!black},
		ylabel={Trajectory length},
		label style={font=\LARGE},
		tick label style={font=\Large},
		ymajorgrids,
		ytick style={color=black},
		ymin=305,
		ymax=355,
		xtick={10, 12, 15}
		]
		\addplot [mark=*, mark size=5, mark options={solid}]
		table {%
			10 316.25
			12 324.8
			15 346.12
		};
	\end{axis}
	
\end{tikzpicture}
	}
	\caption{Average trajectory length with heuristic control~\eqref{eq:control-heuristic} when increasing robot density in test ($R=10$ in training).}
	\label{fig:performance_scalability}
\end{figure*}

\myParagraph{Goal-Reaching}
While the GNN-based controller $\pi_{\xi}$ provides a satisfactory safety level,
we numerically observe that it hardly outputs zero corrective actions $\upert[t]{i}$.
This causes the robots to settle off their respective goal locations.
Figure~\ref{fig:performance} shows the average robot-to-goal distance after the episode deadline expires under single-integrator model
both with perfect information-exchange (solid line) and realistic information-exchange (marks).
Visual inspection reveals that robots settle close to goals with a steady-state error.
The latter increases with communication delays,
meaning that the robots apply stronger corrections to compensate for more outdated data.
The dashed blue line shows the target distance below which the goal is reached.
For context,
robots are simulated as circles with radius $0.05$ as shown by the red dotted line.

We explore how the heuristic switching controller~\eqref{eq:control-heuristic} described in \autoref{sec:heuristic} can counteract this effect.
Recall that this heuristics may not retain any safety guarantee of the GNN-based distributed controller $\controlmodel$.
For this experiment,
we deploy the controller trained with delay coefficient $c_\text{del} = 0.5$ in all tests and set $\varphi=\{0.8,1,1.1\}$ in~\eqref{eq:control-heuristic} to test with $c_\text{del}=\{0.2,0.5,0.8\}$,
respectively,
to reflect the different delay distributions.
Using~\eqref{eq:control-heuristic},
the robots almost always reach their respective goals before the deadline.
The average trajectory length (number of steps) from start to goal using controller~\eqref{eq:control-heuristic} are shown in~\cref{fig:performance-heuristics}.
The trajectory lengths increase with delays,
as expected because robots need to act more cautiously.
On the other hand,
\cref{fig:safety-heuristics} shows that the safety rates are lower than those obtained with the trained controller and shown in~\autoref{fig:safety}.
This is because the heuristic control~\eqref{eq:control-heuristic} sacrifices safety to achieve better goal-reaching performance.
In particular,
the predictor introduces extra uncertainty when estimating the CBF derivative in~\eqref{eq:control-heuristic},
which causes safety degradation especially with long delays.

To get a sense of the performance loss due to communication delays,
the dashed line in \cref{fig:performance-heuristics} shows the average trajectory length obtained under perfect-information exchange
and by computing the CBF derivative in~\eqref{eq:control-heuristic} in a centralized fashion.\footnote{
	The delay coefficients on the $x$-axis are irrelevant for this plot since instantaneous communication is implemented.
}
This case is however impractical either under communication delays or if the robots do not know each other's goals
(which here are sufficient to compute nominal control inputs),
and we show it just for comparison as upper bound for performance.

\myParagraph{Scalability for realistic information-exchange}
Differently from the perfect-information exchange,
for which we study scalability in \autoref{fig:safety},
increasing the number of robots can negatively affect scalability in the presence of communication delays.
In fact,
by increasing \emph{density} of robots in the workspace,
model~\eqref{eq:comm-delays-sim} generally produces longer delays because robots communicate with more neighbors.
This means that increasing the number of robots also causes a distribution shift from delays observed during training.

Figure~\ref{fig:safety_scalability} shows the safety rates obtained in test under single integrator model when increasing the robot density (\ie number of robots with fixed workspace) from training settings.
We show safety rates with both the GNN-based controller~\eqref{eq:controller-gnn} (label ``GNN controller'')
and the heuristic controller~\eqref{eq:control-heuristic} (label ``Heuristic'').
For ``GNN controller'',
we deploy the models trained with delay coefficient $c_\text{del}=0.5$ in the tests with $c_\text{del}=0.2$ and with $c_\text{del}=0.5$,
while we deploy the models trained with $c_\text{del}=0.8$ in the test with $c_\text{del}=0.8$.
For ``Heuristic'',
we use only the models trained with $c_\text{del}=0.5$.
While the controller is able to generalize till double the robot density experienced in training when tested with $c_\text{del}=0.2$,
a safety degradation can be observed with the other two delay coefficients as the system is made more cluttered.
This is not surprising because,
while acceptable safety levels are maintained for (slightly) denser scenarios,
the controller cannot generalize to arbitrarily longer delays.
Moreover,
we see that ``Heuristic'' provides smaller safety rates because it lacks statistical guarantees and pushes on performance.
Therefore,
this study suggests that a careful assessment of communication delays is crucial to train the predictor-controller models in an effective manner.

We also study how performance changes as robot density increases.
Figure~\ref{fig:performance_scalability} shows the average lengths of trajectories run by robots to reach their goals under the heuristic control~\eqref{eq:control-heuristic},
analogously to the black circles in~\cref{fig:performance-heuristics}.
More cluttered environments force robots to navigate longer paths to reach their goals,
which agrees with intuition.
	

\section{Limitations and Future Research}
\label{sec:limitations}

While our theoretical results and methodology are quite general, 
we have demonstrated it only for the task of collision avoidance under single integrator and Dubins car dynamic models.
Showing applicability of our methodology to other tasks and dynamics is an important research avenue to be explored in future work.
In particular,
we have identified the prediction as a serious bottleneck for safety under realistic information-exchange for more complicated dynamics,
which urges more research to apply the proposed framework to real-world applications.
This could involve a different predictor architecture or embedding prediction uncertainty,
such as via Gaussian Processes or Stochastic Barrier Functions~\cite{Mathiesen23lcss-neuralBarrierStochastic}.
Alternatively,
a different predictor-controller interplay may be thought so as to enhance prediction capabilities via a tighter coordination of control actions.
Another point that the present work does not address is formal verification of the leaned distributed CBF,
which is crucial for deployment in safety-critical applications.
\revision{Finally,
the promising results with the predictor-based framework open a research avenue to controlling general delay systems,
and could be integrated with actuation delays in either model- or learning-based fashion.}

\section{Conclusion}
\label{sec:conclusion}

We have proposed distributed control barrier functions for networked autonomous systems
and studied a learning-based approach based on GNN models to jointly learn a distributed CBF and a safe distributed controller.
We have empirically shown that communication delays can disrupt safety certification if not accounted for,
and that a predictor-based approach where a predictor model is alternatively trained with the CBF and controller in the presence of delays can restore safety.
In particular,
using AoI of exchanged information has proven useful to handle delays.
This allows for a simple, yet principled, approach to solve other control problems with delays,
which are typically a challenge in controller design.

	
	\appendices

\section{Proof of \cref{thm:distr-cbf}}
\label[appendix]{app:proof}

Characterization~\eqref{eq:safe-set-local} follows directly from~\cref{ass:safe-set,ass:dist} and the definition of $\neigh{i}$. 
In view of~\cref{def:distributed-cbf} and \cref{ass:dist},
if all control inputs are such that $\un{i}\in\calU_\safe(\xn{i}) \ \forall i$ and as long as no neighborhood change,
then $h$ is differentiable and $\safe$ is forward invariant by~\cref{thm:cbf}.
We need to prove that the $\safe$ is forward invariant when neighborhoods change.
Assume $\rho(\vp_i-\vp_j)>0$ (\ie $j\notin\neigh{i}$) at time $t$ and $\rho(\vp_i-\vp_j)\le0$ (\ie $j\in\neigh{i}$) at time $t+\epsilon$ for some $\epsilon>0$.
Because $\rho$ and $\vp$ are continuous, 
there exists $\epsilon_1 \in (0,\epsilon)$ such that $\rho(\vp_i-\vp_j) = 0$ at time $t+\epsilon_1$.
Thus,
$j\in\neigh{i}$ at time $t+\epsilon_1$ and,
by \cref{ass:dist},
if $h(\xn{i}) \ge 0 $ at time $t$ and $\un{i}\in\calU_\safe(\xn{i})$ for every $t'\in[t,t+\epsilon_1)$,
then $h(\xn{i})\ge0$ at time $t+\epsilon_1$. 
Hence,
the set $\neigh{i}\cup\{i\}$ with $j\in\neigh{i}$ complies with $\vxx\in\safe$ at time $t+\epsilon_1$ and condition~\eqref{eq:control-distributed-cbf} applied to this new neighborhood ensures $h(\xn{i})\ge0$ for $t'\ge t+\epsilon_1$.

\section{Learning Architecture used for Experiments}
\label[appendix]{app:exp-models}

\myParagraph{Predictor}
The predictor $\predmodel$ estimates current states differences $\predval{ij}{t}$ between robots $i$ and $j$,
which are fed to the distributed CBF model $\cbfmodel$ and distributed controller model $\controlmodel$.
To leverage data streams received from neighbors,
we parametrize the predictor as a recurrent neural network to learn correlations across message sequences in $\datarel{i}{j}(t)$.
We use PyTorch \texttt{LSTM} with four layers,
hidden size $256$,
and dropout probability $0.1$.
Also,
we discard all messages with AoI greater than $\aoimax = 5$ steps for the single integrator model and $\aoimax=10$ steps for the Dubins Car model.

\myParagraph{CBF and Controller}
We design $\cbfmodel$ and $\controlmodel$ as compositions of one GNN layer followed by an MLP.
In the perfect information-exchange setting,
when current relative state information $\xn[t]{i}$ is readily available,
we set
\begin{gather}
	\cbfval[\paramcbf]{\xn[t]{i}} = \mlp_h\lr\feat{i}{(K)}(t)\rr\label{eq:cbf-cascade-1}\\
	\control{\xn[t]{i}} = \mlp_\pi\lr\feat{i}{(K)}(t), \uref[t]{i}\rr\label{eq:controller-cascade-1},
\end{gather}
where we set $K=1$ and the GNN model $\gnn$ extracts features at time $t$ as (\cf~\eqref{eq:message-passing})
\begin{equation}\label{eq:message-passing-sim}
	\begin{aligned}
		\feat{i}{(k+1)}(t) &=\gamma \lr \feat{i}{(k)}(t), \aggr_{j\in\robset : \data{i}{j}(t)\neq\emptyset} \phi\lr\xn[t]{ij}\rr \rr\\
		\feat{i}{(K)}(t) &= \gnn\lr\xn[t]{i}\rr.
	\end{aligned}
\end{equation}
The models $\mlp_h$ and $\mlp_\pi$ are MLPs that map the extracted features to CBF values and corrective control inputs,
respectively.
Under realistic information-exchange,
the relative states $\xn[t]{ij}$ are replaced by their estimated values $\predval{ij}{t}$ computed by the predictor.
We also feed the reference signal $\uref[t]{i}$ to the controller module $\mlp_\pi$ to enhance performance.
Because we embed within the (estimated) relative states $\xn[t]{i}$ into the edge features between nodes,
we set the initial robot features $\feat{i}{(0)}(t)$ simply as vector of one's for all robots.\footnote{
	While robot-specific information might be included in $\feat{i}{(0)}(t)$,
	exploring this design option is out of scope of the present contribution.}

In our implementation of the GNN model $\gnn$,
we parametrize both $\phi$ and $\gamma$ in~\eqref{eq:message-passing-sim} as MLPs with two hidden layers,
each with $2048$ nodes,
and $\mathrm{ReLu}$ activation functions.
We use an attention-type aggregation function $\aggr$ because,
intuitively,
not all neighbors carry the same importance to safety -- closer neighbors have more chance of colliding than far-away ones.
Specifically,
we use \texttt{AttentionalAggregation}~\cite{Li19icml-attentionalAggregation} whereby the relevance of each neighbor is parametrized via an MLP
(that we implement with two hidden layers with $128$ nodes each and $\mathrm{ReLu}$ activation),
followed by softmax weighing of the neighbors' scores.
Finally,
we parametrize the MLPs $\mlp_h$ and $\mlp_\pi$ in~\eqref{eq:controller-cascade-1}--\eqref{eq:cbf-cascade-1} with three hidden layers respectively with $512$, $128$, and $32$ nodes,
all with $\mathrm{ReLu}$ activation except for the $\mathrm{tanh}$ output activation of $\mlp_h$ that bounds the CBF values within the interval $[0,1]$.

\begin{table}
	\begin{center}
		\small
		\begin{tabular}{lccccccc}
			\toprule
			&$c_\text{del}$ 	& $w_\safe$	& $w_{\unsafe}$	& $w_\text{der}$	& $\alpha$	& $\epsilon$	& $w_\text{contr}$\\
			\midrule
			exp. 1	& $0.5$				& $1$		& $1$			& $0.7$				& $1$		& $0.02$ 		& $0.001$	\\
			exp. 2	& $0.8$				& $0.9$		& $1$			& $0.7$				& $1$		& $0.02$ 		& $0.0005$	\\
			\bottomrule
		\end{tabular}
	\end{center}
	\caption{Parameters and hyperparameters used with single integrator model.}
	\label{table:hyperparams-si}
\end{table}

\begin{table}
	\begin{center}
		\small
		\begin{tabular}{cccccc}
			\toprule
			$w_\safe$	& $w_{\unsafe}$	& $w_\text{der}$	& $\alpha$	& $\epsilon$	& $w_\text{contr}$\\
			\midrule
			$1$			& $1.2$					& $0.5$					& $1$			& $0.02$				& $0.001$	\\
			\bottomrule
		\end{tabular}
	\end{center}
	\caption{Parameters and hyperparameters used with Dubins car model.}
	\label{table:hyperparams-dc}
\end{table}

Given a batch with $S$ samples,
each learning iteration updates the predictor model with the batch loss
\begin{equation}\label{eq:loss-predictor-batch}
	\losspred^S(\parampred) = \frac{1}{S}\sum_{s=1}^S\losspred(\parampred;\x{}^s,\calI^s)
\end{equation}
and the distributed CBF and controller models are jointly trained with the batch loss
\begin{equation}\label{eq:loss-ctrl-batch}
	\calL^S(\paramcbf,\paramcontr) = \frac{1}{S}\sum_{s=1}^S \calL(\paramcbf,\paramcontr;\x{}^s,\calI^s),
\end{equation}
where $\losspred$ and $\calL$ are respectively defined in~\eqref{eq:loss-pred} and~\eqref{eq:loss-cbf-controller},
and the $s$th sample is given by the pair $(\x{}^s,\calI^s)$ that contains respectively the state of the system and received data $\data{i}{j}$ of all robots $i$ for all neighbors $j$.
In the loss term related to the distributed CBF~\eqref{eq:loss-cbf},
we set a linear class $\calK$ function $\alpha(x) = \alpha x$ with slope $\alpha>0$.
Moreover,
we numerically approximate the CBF derivative (that we use only in training) in~\eqref{eq:loss-cbf} via~\eqref{eq:cbf-derivative}.
Parameters and learning hyperparameters used for training are reported in~\cref{table:hyperparams-si,table:hyperparams-dc}.
Again,
when training together predictor model and distributed CBF and distributed controller models,
we replace exact values of $\xn[t]{ij}$ with their predicted values $\predval{ij}{t}$ when computing the distributed CFB $\cbfmodel(\predval{i}{t})$ and the control input $\control{\predval{i}{t}}$,
but we use the ground truth (\ie actual values of $\xn[t]{i}$ and of $\xn[t+T_\text{s}]{i}$) to compute the time derivative of the distributed CBF via~\eqref{eq:cbf-derivative}.
	


\begin{thebibliography}{10}
\providecommand{\url}[1]{#1}
\csname url@samestyle\endcsname
\providecommand{\newblock}{\relax}
\providecommand{\bibinfo}[2]{#2}
\providecommand{\BIBentrySTDinterwordspacing}{\spaceskip=0pt\relax}
\providecommand{\BIBentryALTinterwordstretchfactor}{4}
\providecommand{\BIBentryALTinterwordspacing}{\spaceskip=\fontdimen2\font plus
\BIBentryALTinterwordstretchfactor\fontdimen3\font minus
  \fontdimen4\font\relax}
\providecommand{\BIBforeignlanguage}[2]{{%
\expandafter\ifx\csname l@#1\endcsname\relax
\typeout{** WARNING: IEEEtran.bst: No hyphenation pattern has been}%
\typeout{** loaded for the language `#1'. Using the pattern for}%
\typeout{** the default language instead.}%
\else
\language=\csname l@#1\endcsname
\fi
#2}}
\providecommand{\BIBdecl}{\relax}
\BIBdecl

\bibitem{Hu21tvt-coordinatedControlMobileRobots}
J.~Hu, P.~Bhowmick, and A.~Lanzon, ``Group {{Coordinated Control}} of
  {{Networked Mobile Robots With Applications}} to {{Object Transportation}},''
  \emph{IEEE Trans. Veh. Technol.}, vol.~70, no.~8, pp. 8269--8274, 2021.

\bibitem{Li24tvt-flockingMultiRobot}
S.~Li, S.~Zhang, G.~He, and T.~Jiang, ``Discrete-{{Time Flocking Control}} in
  {{Multi-Robot Systems With Random Link Failures}},'' \emph{IEEE Trans. Veh.
  Technol.}, vol.~73, no.~9, pp. 12\,290--12\,304, 2024.

\bibitem{Zhou19tvt-rlMultiRobot}
X.~Zhou, W.~Wang, T.~Wang, Y.~Lei, and F.~Zhong, ``Bayesian {{Reinforcement
  Learning}} for {{Multi-Robot Decentralized Patrolling}} in {{Uncertain
  Environments}},'' \emph{IEEE Trans. Veh. Technol.}, vol.~68, no.~12, pp.
  11\,691--11\,703, 2019.

\bibitem{Hu20tvt-voronoiMultiRobot}
J.~Hu, H.~Niu, J.~Carrasco, B.~Lennox, and F.~Arvin, ``Voronoi-{{Based
  Multi-Robot Autonomous Exploration}} in {{Unknown Environments}} via {{Deep
  Reinforcement Learning}},'' \emph{IEEE Trans. Veh. Technol.}, vol.~69,
  no.~12, pp. 14\,413--14\,423, 2020.

\bibitem{Lee13tm-teleoperation}
D.~Lee, A.~Franchi, H.~I. Son, C.~Ha, H.~H. B{\"u}lthoff, and P.~R. Giordano,
  ``Semiautonomous {{Haptic Teleoperation Control Architecture}} of {{Multiple
  Unmanned Aerial Vehicles}},'' \emph{IEEE/ASME Trans. Mechatronics}, vol.~18,
  no.~4, pp. 1334--1345, 2013.

\bibitem{Pezzutto22lcss-remoteMPC}
M.~Pezzutto, M.~Farina, R.~Carli, and L.~Schenato, ``Remote {{MPC}} for
  {{Tracking Over Lossy Networks}},'' \emph{IEEE Contr. Syst. Lett.}, vol.~6,
  pp. 1040--1045, 2022.

\bibitem{Capelli21icra-connectivityCBFDelays}
B.~Capelli, H.~Fouad, G.~Beltrame, and L.~Sabattini, ``Decentralized
  {{Connectivity Maintenance}} with {{Time Delays}} using {{Control Barrier
  Functions}},'' in \emph{Proc. {{IEEE ICRA}}}, 2021, pp. 1586--1592.

\bibitem{Ames17tac-cbf}
A.~D. Ames, X.~Xu, J.~W. Grizzle, and P.~Tabuada, ``Control {{Barrier Function
  Based Quadratic Programs}} for {{Safety Critical Systems}},'' \emph{IEEE
  Trans. Autom. Control}, vol.~62, no.~8, pp. 3861--3876, 2017.

\bibitem{Wang17tro-safeMultirobotCBF}
L.~Wang, A.~D. Ames, and M.~Egerstedt, ``Safety {{Barrier Certificates}} for
  {{Collisions-Free Multirobot Systems}},'' \emph{IEEE Trans. Robot.}, vol.~33,
  no.~3, pp. 661--674, 2017.

\bibitem{Chen21lcss-cbfMultiAgent}
Y.~Chen, A.~Singletary, and A.~D. Ames, ``Guaranteed {{Obstacle Avoidance}} for
  {{Multi-Robot Operations With Limited Actuation}}: {{A Control Barrier
  Function Approach}},'' \emph{IEEE Control Syst. Lett.}, vol.~5, no.~1, pp.
  127--132, 2021.

\bibitem{Jankovic18acc-cbfInputDelay}
M.~Jankovic, ``Control {{Barrier Functions}} for {{Constrained Control}} of
  {{Linear Systems}} with {{Input Delay}},'' in \emph{Proc. {{IEEE ACC}}},
  2018, pp. 3316--3321.

\bibitem{Sun19book-AgeInformation}
Y.~Sun, I.~Kadota, R.~Talak, and E.~Modiano, ``Age of {{Information}}: {{A New
  Metric}} for {{Information Freshness}},'' \emph{Synthesis Lectures on
  Communication Networks}, vol.~12, no.~2, pp. 1--224, 2019.

\bibitem{Molnar23tcst-cbfInputDelay}
T.~G. Molnar, A.~K. Kiss, A.~D. Ames, and G.~Orosz, ``Safety-{{Critical Control
  With Input Delay}} in {{Dynamic Environment}},'' \emph{IEEE Trans. Contr.
  Syst. Technol.}, vol.~31, no.~4, pp. 1507--1520, 2023.

\bibitem{Hamdipoor23ejc-environmentallyRobustCBF}
V.~Hamdipoor, N.~Meskin, and C.~G. Cassandras, ``Safe control synthesis using
  environmentally robust control barrier functions,'' \emph{European J.
  Control}, vol.~74, p. 100840, 2023.

\bibitem{Breeden22cdc-proactiveCBF}
J.~Breeden and D.~Panagou, ``Predictive {{Control Barrier Functions}} for
  {{Online Safety Critical Control}},'' in \emph{Proc. {{IEEE CDC}}}, 2022, pp.
  924--931.

\bibitem{Agrawal23lcss-observerControllerCBF}
D.~R. Agrawal and D.~Panagou, ``Safe and {{Robust Observer-Controller Synthesis
  Using Control Barrier Functions}},'' \emph{IEEE Contr. Syst. Lett.}, vol.~7,
  pp. 127--132, 2023.

\bibitem{Panagou16tac-multiRobotLyapunovBarrierFunction}
D.~Panagou, D.~M. Stipanovic, and P.~G. Voulgaris, ``Distributed {{Coordination
  Control}} for {{Multi-Robot Networks Using Lyapunov-Like Barrier
  Functions}},'' \emph{IEEE Trans. Autom. Control}, vol.~61, no.~3, pp.
  617--632, 2016.

\bibitem{Cavorsi22rss-resilienceCBF}
M.~Cavorsi, B.~Capelli, L.~Sabattini, and S.~Gil, ``Multi-{{Robot Adversarial
  Resilience}} using {{Control Barrier Functions}},'' in \emph{Proc. {{Robot}}.
  {{Sci}}. {{Syst}}.}, vol.~18, 2022.

\bibitem{Ong23automatica-nonsmoothCBFConnectivityMaintenance}
P.~Ong, B.~Capelli, L.~Sabattini, and J.~Cort{\'e}s, ``Nonsmooth {{Control
  Barrier Function}} design of continuous constraints for network connectivity
  maintenance,'' \emph{Automatica}, vol. 156, p. 111209, 2023.

\bibitem{Xiao23tro-BarrierNet}
W.~Xiao, T.-H. Wang, R.~Hasani, M.~Chahine, A.~Amini, X.~Li, and D.~Rus,
  ``{{BarrierNet}}: {{Differentiable Control Barrier Functions}} for
  {{Learning}} of {{Safe Robot Control}},'' \emph{IEEE Trans. Robot.}, vol.~39,
  no.~3, pp. 2289--2307, 2023.

\bibitem{Gaby22cdc-Lyapunov-Net}
N.~Gaby, F.~Zhang, and X.~Ye, ``Lyapunov-{{Net}}: {{A Deep Neural Network
  Architecture}} for {{Lyapunov Function Approximation}},'' in \emph{Proc.
  {{IEEE CDC}}}, 2022, pp. 2091--2096.

\bibitem{Abate21lcss-formalSynthesisLyapunovNN}
A.~Abate, D.~Ahmed, M.~Giacobbe, and A.~Peruffo, ``Formal {{Synthesis}} of
  {{Lyapunov Neural Networks}},'' \emph{IEEE Contr. Syst. Lett.}, vol.~5,
  no.~3, pp. 773--778, 2021.

\bibitem{Abate21ichs-FOSSIL}
A.~Abate, D.~Ahmed, A.~Edwards, M.~Giacobbe, and A.~Peruffo, ``{{FOSSIL}}: A
  software tool for the formal synthesis of lyapunov functions and barrier
  certificates using neural networks,'' in \emph{Proc. {{Int}}. {{Conf}}.
  {{Hybrid Syst}}. {{Comput}}. {{Control}}}, 2021, pp. 1--11.

\bibitem{Zakwan23lcss-hamiltonianNN}
M.~Zakwan, M.~{d'Angelo}, and G.~{Ferrari-Trecate}, ``Universal {{Approximation
  Property}} of {{Hamiltonian Deep Neural Networks}},'' \emph{IEEE Contr. Syst.
  Lett.}, vol.~7, pp. 2689--2694, 2023.

\bibitem{Furieri22cdc-neuralSystemLevelSynthesis}
L.~Furieri, C.~L. Galimberti, and G.~{Ferrari-Trecate}, ``Neural {{System Level
  Synthesis}}: {{Learning}} over {{All Stabilizing Policies}} for {{Nonlinear
  Systems}},'' in \emph{Proc. {{IEEE CDC}}}, 2022, pp. 2765--2770.

\bibitem{Mathiesen23lcss-neuralBarrierStochastic}
F.~B. Mathiesen, S.~C. Calvert, and L.~Laurenti, ``Safety {{Certification}} for
  {{Stochastic Systems}} via {{Neural Barrier Functions}},'' \emph{IEEE Control
  Syst. Lett.}, vol.~7, pp. 973--978, 2023.

\bibitem{Sinopoli04tac-KalmanIntermittent}
B.~Sinopoli, L.~Schenato, M.~Franceschetti, K.~Poolla, M.~Jordan, and
  S.~Sastry, ``Kalman {{Filtering With Intermittent Observations}},''
  \emph{IEEE Trans. Automat. Contr.}, vol.~49, no.~9, pp. 1453--1464, 2004.

\bibitem{Munz10automatica-consensusDelays}
U.~M{\"u}nz, A.~Papachristodoulou, and F.~Allg{\"o}wer, ``Delay robustness in
  consensus problems,'' \emph{Automatica}, vol.~46, no.~8, pp. 1252--1265,
  2010.

\bibitem{Matni17tcns-h2controlAtomicNorm}
N.~Matni, ``Communication {{Delay Co-Design}} in $\mathcal{H}_2$-{{Distributed
  Control Using Atomic Norm Minimization}},'' \emph{IEEE Trans. Control Netw.
  Syst.}, vol.~4, no.~2, pp. 267--278, 2017.

\bibitem{Gomez19tac-h2controlDelays}
M.~A. Gomez, A.~V. Egorov, S.~Mondi{\'e}, and W.~Michiels, ``Optimization of
  the $\mathcal{H}_2$ {{Norm}} for {{Single-Delay Systems}}, {{With
  Application}} to {{Control Design}} and {{Model Approximation}},'' \emph{IEEE
  Trans. Autom. Control}, vol.~64, no.~2, pp. 804--811, 2019.

\bibitem{Branz22tcst-Drive-by-Wi-Fi}
F.~Branz, R.~Antonello, M.~Pezzutto, S.~Vitturi, F.~Tramarin, and L.~Schenato,
  ``Drive-by-{{Wi-Fi}}: {{Model-Based Control Over Wireless}} at 1 {{kHz}},''
  \emph{IEEE Trans. Control Syst. Technol.}, vol.~30, no.~3, pp. 1078--1089,
  2022.

\bibitem{Tripathi23tmc-AgeOptimal}
V.~Tripathi, R.~Talak, and E.~Modiano, ``Age {{Optimal Information Gathering}}
  and {{Dissemination}} on {{Graphs}},'' \emph{IEEE Trans. Mobile Comput.},
  vol.~22, no.~1, pp. 54--68, 2023.

\bibitem{Talak20tn-aoiInterference}
R.~Talak, S.~Karaman, and E.~Modiano, ``Optimizing {{Information Freshness}} in
  {{Wireless Networks Under General Interference Constraints}},''
  \emph{IEEE/ACM Trans. Netw.}, vol.~28, no.~1, pp. 15--28, 2020.

\bibitem{Tripathi19allerton-whittle}
V.~Tripathi and E.~Modiano, ``A {{Whittle Index Approach}} to {{Minimizing
  Functions}} of {{Age}} of {{Information}},'' in \emph{Proc. {{Allerton
  Conf}}. {{Commun}}. {{Control Comput}}.}, 2019, pp. 1160--1167.

\bibitem{Ayan19-AoIvsVoI}
O.~Ayan, M.~Vilgelm, M.~Kl{\"u}gel, S.~Hirche, and W.~Kellerer,
  ``Age-of-information vs. value-of-information scheduling for cellular
  networked control systems,'' in \emph{Proc. {{ACM}}/{{IEEE ICCPS}}}, 2019,
  pp. 109--117.

\bibitem{Klugel19infocom-aoiNetworkedControl}
M.~Kl{\"u}gel, M.~H. Mamduhi, S.~Hirche, and W.~Kellerer, ``{{AoI-Penalty
  Minimization}} for {{Networked Control Systems}} with {{Packet Loss}},'' in
  \emph{Proc. {{IEEE INFOCOM WKSHPS}}}, 2019, pp. 189--196.

\bibitem{Champati19infocom-aoiNetworkSingleLoop}
J.~P. Champati, M.~H. Mamduhi, K.~H. Johansson, and J.~Gross, ``Performance
  {{Characterization Using AoI}} in a {{Single-loop Networked Control
  System}},'' in \emph{Proc. {{IEEE INFOCOM WKSHPS}}}, 2019, pp. 197--203.

\bibitem{Kiss23ijrnl-cbfDelays}
A.~K. Kiss, T.~G. Molnar, A.~D. Ames, and G.~Orosz, ``Control barrier
  functionals: {{Safety-critical}} control for time delay systems,'' \emph{Int.
  J. Robust Nonlin. Control}, vol.~33, no.~12, pp. 7282--7309, 2023.

\bibitem{Kipf16iclr-graphCOnvolutonalNetworks}
T.~N. Kipf and M.~Welling, ``Semi-{{Supervised Classification}} with {{Graph
  Convolutional Networks}},'' in \emph{Proc. {{ICLR}}}, 2016.

\bibitem{Velickovic18iclr-attentionNetwork}
P.~Veli{\v c}kovi{\'c}, G.~Cucurull, A.~Casanova, A.~Romero, P.~Li{\`o}, and
  Y.~Bengio, ``Graph {{Attention Networks}},'' in \emph{Proc. {{ICLR}}}, 2018.

\bibitem{Brody21iclr-howAttentive}
S.~Brody, U.~Alon, and E.~Yahav, ``How {{Attentive}} are {{Graph Attention
  Networks}}?'' in \emph{Proc. {{ICLR}}}, 2021.

\bibitem{Khan20corl-graphPolicyGradient}
A.~Khan, E.~Tolstaya, A.~Ribeiro, and V.~Kumar, ``Graph {{Policy Gradients}}
  for {{Large Scale Robot Control}},'' in \emph{Proc. {{CoRL}}}, 2020, pp.
  823--834.

\bibitem{Li20iros-gnnDecentralizedControl}
Q.~Li, F.~Gama, A.~Ribeiro, and A.~Prorok, ``Graph {{Neural Networks}} for
  {{Decentralized Multi-Robot Path Planning}},'' in \emph{Proc. {{IEEE}}/{{RSJ
  IROS}}}, 2020, pp. 11\,785--11\,792.

\bibitem{Kortvelesy21icra-ModGNN}
R.~Kortvelesy and A.~Prorok, ``{{ModGNN}}: {{Expert Policy Approximation}} in
  {{Multi-Agent Systems}} with a {{Modular Graph Neural Network
  Architecture}},'' in \emph{Proc. {{IEEE ICRA}}}, 2021, pp. 9161--9167.

\bibitem{Gama22tsp-gnnCOntroller}
F.~Gama, Q.~Li, E.~Tolstaya, A.~Prorok, and A.~Ribeiro, ``Synthesizing
  {{Decentralized Controllers With Graph Neural Networks}} and {{Imitation
  Learning}},'' \emph{IEEE Trans. Signal Process.}, vol.~70, pp. 1932--1946,
  2022.

\bibitem{Sebastian23icra-LEMURS}
E.~Sebasti{\'a}n, T.~Duong, N.~Atanasov, E.~Montijano, and C.~Sag{\"u}{\'e}s,
  ``{{LEMURS}}: {{Learning Distributed Multi-Robot Interactions}},'' in
  \emph{Proc. {{IEEE ICRA}}}, 2023, pp. 7713--7719.

\bibitem{Rossi23automatica-sampledCommunication}
E.~Rossi, M.~Tognon, L.~Ballotta, R.~Carli, J.~Cort{\'e}s, A.~Franchi, and
  L.~Schenato, ``Coordinated multi-robot trajectory tracking control over
  sampled communication,'' \emph{Automatica}, vol. 151, p. 110941, 2023.

\bibitem{Li19icml-attentionalAggregation}
Y.~Li, C.~Gu, T.~Dullien, O.~Vinyals, and P.~Kohli, ``Graph {{Matching
  Networks}} for {{Learning}} the {{Similarity}} of {{Graph Structured
  Objects}},'' in \emph{Proc. {{ICML}}}, 2019, pp. 3835--3845.

\end{thebibliography}
	

\begin{IEEEbiography}[{\includegraphics[width=1in,height=1.25in,clip,keepaspectratio]{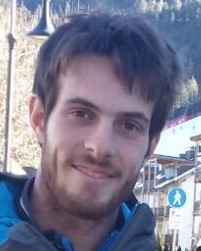}}]{Luca Ballotta}
	received the Master's degree in automation engineering and the Ph.D. degree in information engineering from the University of Padova, Padova, Italy, in 2019 and 2023,
respectively.
He is currently a Postdoctoral Researcher with the Delft Center for Systems and Control (DCSC), Delft University of Technology, Delft, Netherlands.
In 2020 and 2022,
he was Visiting Student at the Massachusetts Institute of Technology, Cambridge, MA, USA.
He was the recipient of the Young Author Prize at the 2020 IFAC World Congress and was also the Finalist of the 2024 EECI Ph.D. Award.
His research interests include multi-agent systems and networked control under delays, 
resilient distributed control and distributed learning,
safe control,
and control under sparsity constraints.
\end{IEEEbiography}

\begin{IEEEbiography}[{\includegraphics[width=1in,height=1.25in,clip,keepaspectratio]{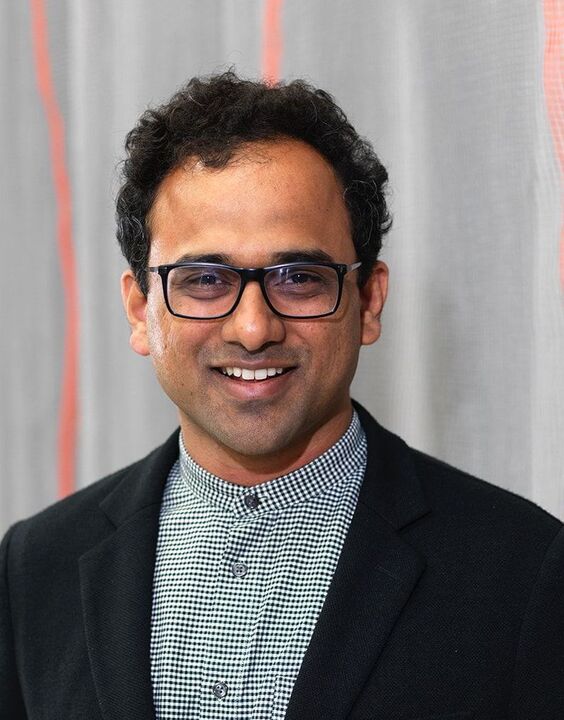}}]{Rajat Talak}
	(Member, IEEE) received the B.Tech. Degree from the Department of Electronics and Communication Engineering, National Institute of Technology Karnataka, Surathkal, India, in 2010, the Master of Science degree from the Department of Electrical Communication Engineering, Indian Institute of Science, Bengaluru, India, in 2013, and the Ph.D. degree from the Laboratory of Information and Decision Systems, Massachusetts Institute of Technology (MIT), Cambridge, MA, USA, in 2020. He is currently a Research Scientist at the Department of Aeronautics and Astronautics, MIT. Prior to this, he was a Post-Doctoral Associate with the Department of Aeronautics and Astronautics, MIT. 
He was the recipient of the Best Paper Award at the ACM MobiHoc 2018 and of the Gold medal for his Master’s thesis.
His research interests are in the field of robot perception, optimization and learning, autonomous systems, and communication networks.
\end{IEEEbiography}
	
\end{document}